\documentclass{article} 
\usepackage{url}
\usepackage{amsmath,amssymb,bm,amsthm}
\usepackage{graphicx} 
\usepackage{subfigure}
\usepackage{algorithm}
\usepackage{algorithmic}
\usepackage{hyperref}
\usepackage[numbers]{natbib}

\def\a{{\bf a}}
\def\b{{\bf b}}
\def\c{{\bf c}}

\def\p{{\bf p}}

\def\s{{\bf s}}

\def\u{{\bf u}}
\def\v{{\bf v}}
\def\w{{\bf w}}
\def\x{{\bf x}}
\def\y{{\bf y}}
\def\z{{\bf z}}

\def\A{{\bf A}}
\def\B{{\bf B}}

\def\G{{\bf G}}
\def\H{{\bf H}}
\def\I{{\bf I}}

\def\0{{\bf 0}}
\def\1{{\bf 1}}
\def\2{{\bf 2}}
\def\3{{\bf 3}}
\def\4{{\bf 4}}
\def\5{{\bf 5}}
\def\6{{\bf 6}}
\def\7{{\bf 7}}
\def\8{{\bf 8}}
\def\9{{\bf 9}}

\def\LM{{\mathcal L}}

\def\XM{{\mathcal X}}

\def\EB{{\mathbb E}}

\def\RB{{\mathbb R}}

\def\balpha { {\bm \alpha}}
\def\bbeta { {\bm \beta}}

\newtheorem{theorem}{Theorem}
\newtheorem{lemma}{Lemma}

\begin{document}
\title{Scalable Stochastic Alternating Direction Method of Multipliers}

\author{Shen-Yi Zhao 
Wu-Jun Li 
Zhi-Hua Zhou \\
Department of Computer Science \\
National Key Laboratory for Novel Software Technology \\
Nanjing University, China
}

\maketitle

\begin{abstract}
Most stochastic ADMM~(alternating direction method of multipliers) methods can only achieve a convergence rate which is slower than $O(1/T)$ on general convex problems, where $T$ is the number of iterations. Hence, these methods are not scalable in terms of convergence rate~(computation cost). There exists only one stochastic method, called \mbox{SA-ADMM}, which can achieve a convergence rate of $O(1/T)$ on general convex problems. However, an extra memory is needed for SA-ADMM to store the historic gradients on all samples, and thus it is not scalable in terms of storage cost. In this paper, we propose a novel method, called \underline{sca}lable \underline{s}tochastic \mbox{ADMM}~(\mbox{SCAS-ADMM}), for large-scale optimization and learning problems. Without the need to store the historic gradients on all samples, \mbox{SCAS-ADMM} can achieve the same convergence rate of $O(1/T)$ as the best stochastic method \mbox{SA-ADMM} and batch ADMM on general convex problems. Experiments on graph-guided fused lasso show that SCAS-ADMM can achieve state-of-the-art performance in real applications.
\end{abstract}

\section{Introduction}
The alternating direction method of multipliers~(\mbox{ADMM})~\cite{DBLP:journals/ftml/BoydPCPE11} is proposed to solve the problems which can be formulated as follows:
\begin{align} \label{eq:admm}
\min_{\x,\y}~&P(\x,\y)=f(\x) + g(\y) \\
s.t.\quad~&\A\x + \B\y = \c, \nonumber
\end{align}
where $f(\cdot)$ and $g(\cdot)$ are convex functions, $\A\in\RB^{l\times p}$ and $\B\in\RB^{l\times q}$ are matrices, $\c\in\RB^{l}$ is a vector, $\x\in\RB^{p}$ and $\y\in\RB^{q}$ are variables to be optimized~(learned). By splitting the objective function $P(\cdot)$ into two parts $f(\cdot)$ and $g(\cdot)$, ADMM provides a flexible framework to handle many optimization problems. For example, by taking $f(\x)$ to be the square loss or logistic loss on the training set, $g(\y)$ to be the \mbox{$L_1$-norm} and the constraint to be $\x - \y = \0$, we can get the well-known lasso formulation~\cite{DBLP:journals/jrssb/TibshiraniRLasso96}. Similarly, we can take more complex constraints than that in lasso to get more complex regularization problems such as the structured sparse regularization problems~\cite{DBLP:conf/icml/Suzuki13,DBLP:conf/icml/ZhongK14}. Compared with other optimization methods such as gradient decent, ADMM has demonstrated better performance in many complex regularization problems~\cite{DBLP:conf/icml/Suzuki13,DBLP:conf/icml/ZhongK14}. Furthermore, ADMM can be easily adapted to solve large-scale distributed problems~\cite{DBLP:journals/ftml/BoydPCPE11}. Hence, \mbox{ADMM} has been widely used in a large variety of areas~\cite{DBLP:journals/ftml/BoydPCPE11}.

Deterministic~(batch) ADMM needs to visit all the samples in each iteration. Existing works have shown that batch \mbox{ADMM} is not efficient enough for big data applications with a large amount of training samples~\cite{DBLP:conf/icml/WangB12,DBLP:conf/icml/OuyangHTG13}. Stochastic~(online) ADMM, which visits only one sample or a mini-batch of samples each time, has recently been proved to achieve better performance than batch ADMM~\cite{DBLP:conf/icml/WangB12,DBLP:conf/icml/OuyangHTG13,DBLP:conf/icml/Suzuki13,DBLP:conf/icml/AzadiS14,DBLP:conf/icml/ZhongK14}.
Hence, stochastic ADMM has become a hot research topic and attracted much attention~\cite{DBLP:conf/icml/AzadiS14,DBLP:conf/icml/ZhongK14}.

Online alternating direction method~(OADM)~\cite{DBLP:conf/icml/WangB12} is the first online ADMM method. There is only regret analysis in OADM, based on which we can find that if OADM is adapted for stochastic settings with finite samples, the convergence rate of \mbox{OADM} is $O(1/\sqrt{T})$ for general convex problems where $f(\x)$ and $g(\y)$ are convex but not necessarily to be strongly convex. Here, $T$ is the number of iterations. Besides OADM, several stochastic ADMM methods have been proposed, including stochastic ADMM~(\mbox{STOC-ADMM})~\cite{DBLP:conf/icml/OuyangHTG13}, regularized dual averaging ADMM~(\mbox{RDA-ADMM})~\cite{DBLP:conf/icml/Suzuki13}, online proximal gradient descent based ADMM~(OPG-ADMM)~\cite{DBLP:conf/icml/Suzuki13}, optimal stochastic ADMM~(\mbox{OS-ADMM})~\cite{DBLP:conf/icml/AzadiS14}, and stochastic average ADMM~(SA-ADMM)~\cite{DBLP:conf/icml/ZhongK14}. STOC-ADMM, RDA-ADMM, OPG-ADMM and OS-ADMM achieve a convergence rate of $O(1/\sqrt{T})$ for general convex problems, worse than batch ADMM that has a convergence rate of $O(1/T)$~\cite{DBLP:journals/siamnum/HeY12}. Different from STOC-ADMM, RDA-ADMM, OPG-ADMM and OS-ADMM, SA-ADMM~\cite{DBLP:conf/icml/ZhongK14} can achieve a convergence rate of $O(1/T)$ for general convex problems by using historic gradients to approximate the full gradients in each iteration. Thus, SA-ADMM is the only one which is scalable in terms of convergence rate~(computation cost). However, SA-ADMM requires an extra memory which is typically very large to store the historic gradients on all samples, making it not scalable in terms of storage cost.

In this paper, we propose a novel method, called \underline{sca}lable \underline{s}tochastic \mbox{ADMM}~(\mbox{SCAS-ADMM}), for large-scale optimization and learning problems. The main contributions of \mbox{SCAS-ADMM} are outlined as follows:
\begin{itemize}
\item SCAS-ADMM achieves the same convergence rate of $O(1/T)$ for general convex problems as the best existing stochastic ADMM method~(SA-ADMM) and batch ADMM. Therefore, SCAS-ADMM is scalable in terms of convergence rate~(computation cost).
\item  Different from SA-ADMM, SCAS-ADMM does not need an extra memory to store the historic gradients on all samples. Therefore, SCAS-ADMM is scalable in terms of memory~(storage) cost.
\item Experimental results on graph-guided fused lasso~\cite{DBLP:journals/bioinformatics/KimSX09} show that SCAS-ADMM can achieve state-of-the-art performance in real applications.
\end{itemize}

\vspace{-0.5cm}
\section{Background}\label{sec:background}
\subsection{Convex and Smooth Functions}
We use $\left\| \a \right\|$ to denote the Euclidean~($L_2$) norm of $\a$. A function $h(\cdot)$ is called $\nu'_h$-Lipschitz continuous if: $\exists \nu'_h > 0, ~\forall \a,\b, \hspace{0.2cm} \left\| h(\b) - h(\a) \right\| \leq \nu'_h \left\| \b - \a \right\|$. Assume $h(\cdot)$ is differentiable, and let $\nabla h(\a)$ denote the gradient of $h(\cdot)$ at $\a$. A function $h(\cdot)$ is called \emph{convex} if: $\forall \a,\b, \hspace{0.2cm} h(\b) \geq h(\a) + [\nabla h(\a)]^T (\b-\a)$. Assume $h(\cdot)$ is convex and differentiable.  $h(\cdot)$ is called $\nu_h$-smooth if: $\exists \nu_h > 0, \forall \a,\b,\hspace{0.2cm} h(\b) \leq h(\a) + [\nabla h(\a)]^T (\b-\a) + \frac{\nu_h}{2} \left\| \b - \a \right\|^2$. This is equivalent to say that $\nabla h(\cdot)$ is $\nu_h$-Lipschitz continuous. Here, $\nu_h$ is called the \emph{Lipschits constant} of $h(\cdot)$. A function $h(\cdot)$ is called \emph{strongly convex} if: $\exists \mu_h >0,~\forall \a,\b, \hspace{0.2cm}  h(\b) \geq h(\a) + [\nabla h(\a)]^T (\b-\a) + \frac{\mu_h}{2} \left\| \b - \a \right\|^2$. A function $h(\cdot)$ is called \emph{general convex} if $h(\cdot)$ is convex but not necessarily to be strongly convex.
%
%

%
\subsection{ADMM}
ADMM solves~(\ref{eq:admm}) based on the augmented Lagrangian function:
\begin{align}\label{eq:lagrangian}
L (\x,\y,\bbeta) = &f(\x) + g(\y) + \bbeta^T(\A\x+\B\y-\c) + \frac{\rho}{2} \left\| \A\x + \B\y - \c \right\|^2,
\end{align}
where $\bbeta$ is a vector of Lagrangian multipliers, and $\rho > 0$ is a penalty parameter.


Just like the Gauss-Seidel method, ADMM iteratively updates the variables in an alternating manner as follows~\cite{DBLP:journals/ftml/BoydPCPE11}:
\begin{align}
\x_{t+1} &= \arg\min_\x L(\x,\y_t,\bbeta_t), \label{eq:updataX} \\
\y_{t+1} &= \arg\min_\y L(\x_{t+1},\y,\bbeta_t), \label{eq:updataY} \\
\bbeta_{t+1} &= \bbeta_t + \rho(\A\x_{t+1} + \B\y_{t+1} - \c), \label{eq:updataBeta}
\end{align}
where $\x_t$, $\y_t$ and $\bbeta_t$ denote the values of $\x$, $\y$ and $\bbeta$ at the $t$th iteration, respectively.

In the regularized risk minimization problem which this paper will focus on, the function $f(\x)$ usually has the following structure:
\begin{align}\label{eq:empiricalLoss}
f(\x) = \frac{1}{n} \sum_{i=1}^n f_i(\x),
\end{align}
where $\x$ denotes the model parameter, $n$ is the number of training samples, and each $f_i(\cdot)$ is the empirical loss caused by the $i$th sample. The function $g(\y)$ is usually a regularization term. For example, $f_i(\x) = \log (1+\exp^{-b_i \a_i^T \x})$ in logistic regression~(\mbox{LR}), and $f_i(\x) = (b_i - \a_i^T \x)^2$ in least square, where $(\a_i,b_i)$ is the $i$th training sample with the class label $b_i$. Taking $g(\y) = \left\| \y \right\|_1$ and the constraint $\y = \x$, we can get the lasso formulation~\cite{DBLP:journals/jrssb/TibshiraniRLasso96}. Similarly, we can get more complex regularization problems by taking more complex constraints like $\y=\A\x$.

Unless otherwise stated, $f(\x)$ of the problem we are trying to solve in this paper is defined in~(\ref{eq:empiricalLoss}). Then~(\ref{eq:updataX}) becomes:
\begin{align}\label{eq:updataXwithBatch}
\x_{t+1} = \arg\min_\x& \{\frac{1}{n} \sum_{i=1}^n f_i(\x) + (\bbeta_t)^T(\A\x+\B\y_t-\c) +\frac{\rho}{2}\left\| \A\x+\B\y_t-\c\right\|^2 \}.
\end{align}
From~(\ref{eq:updataXwithBatch}), it is easy to see that ADMM needs to visit all the $n$ samples in each iteration. Hence, this version of ADMM is also called batch ADMM or deterministic ADMM. Some works~\cite{DBLP:conf/icml/WangB12,DBLP:journals/siamnum/HeY12} have proved that the above batch ADMM has a convergence rate $O(1/T)$ for general convex problems where $f(\x)$ and $g(\y)$ are convex but not necessarily to be strongly convex, where $T$ is the number of iterations.

%
Different from batch ADMM, stochastic~(online) ADMM visits only one sample or a mini-batch of samples in each iteration. Recent works have shown that stochastic ADMM can achieve better performance than batch ADMM to handle large-scale datasets in terms of computation complexity and accuracy~\cite{DBLP:conf/icml/WangB12,DBLP:conf/icml/OuyangHTG13}. The computation of~(\ref{eq:updataY}) and~(\ref{eq:updataBeta}) for both batch ADMM and stochastic ADMM are the same, which can typically be easily completed. Hence, different stochastic ADMM methods mainly focus on proposing different solutions for~(\ref{eq:updataXwithBatch}).

\section{Scalable Stochastic ADMM}\label{sec:SCAS-ADMM}

In this section, we present the details of our SCAS-ADMM, which is scalable in terms of both convergence rate and storage cost. Similar to most existing stochastic ADMM methods which adapt stochastic gradient descent~(SGD) or its variants~\cite{DBLP:conf/nips/Xiao09,DBLP:journals/jmlr/DuchiS09,DBLP:conf/nips/RouxSB12,DBLP:conf/icml/Mairal13,DBLP:journals/jmlr/Shalev-Shwartz013} to solve the problem in~(\ref{eq:updataXwithBatch}), SCAS-ADMM is also inspired by an existing SGD method called stochastic variance reduced gradient~(SVRG)~\cite{DBLP:conf/nips/Johnson013}. But different from SVRG, our SCAS-ADMM can be used to model more complex problems with equality constraints.

In this paper, we assume that $f(\cdot)$ and all the $\{f_i(\cdot)\}$ are $v_f$-smooth. For $g(\cdot)$, we only assume it to be convex, but not necessarily to be smooth or Lipschitz continuous. This is a reasonable assumption for many machine learning problems, such as the lasso with logistic loss or square loss. The proof of the theorems of this paper can be found from the Appendix in the supplementary materials.

 %


%
\subsection{General Convex Problems}
In the general convex problems, $f(\cdot)$ is $v_f$-smooth and general convex but not necessarily to be strongly convex.
\subsubsection{Algorithm}
As in existing stochastic ADMM methods~\cite{DBLP:conf/icml/OuyangHTG13,DBLP:conf/icml/ZhongK14}, the update rules for $\y$ and $\bbeta$ are still the same as those in~(\ref{eq:updataY}) and~(\ref{eq:updataBeta}). We only need to design a new strategy to update $\x$. The algorithm for our SCAS-ADMM is briefly presented in Algorithm~\ref{alg:SCAS-ADMM}. It changes~(\ref{eq:updataXwithBatch}) to be:
\begin{align}\label{eq:updateXscas}
\x_{t+1} & = \frac{\sum_{m=0}^{M_t-1} \w_{m}}{M_t},
\end{align}
where $M_t$ is a parameter denoting the number of iterations in the inner loop, and
\begin{align}\label{eq:updateXscasInner}
\w_0 &= \x_t, \nonumber \\
\w_{m+1} &= \pi_\XM(\w_m - \eta_t[\nabla f_{i_m}(\w_m) - \nabla f_{i_m}(\w_0) + \z_t + \A^T\bbeta_t + \rho \A^T(\A\w_m + \B\y_t -\c)]),
\end{align}
with $i_m$ being an index randomly sampled from $\left\{1,2,\cdots,n\right\}$, $\z_t = \nabla f(\x_t)=\frac{1}{n}\sum_{i=1}^n \nabla f_i(\x_t)$ being the full gradient at $\x_t$, $\XM$ being the domain of $\x$, and $\pi_\XM(\cdot)$ denoting the projection operation onto the domain $\XM$.

\begin{algorithm}[htb]
\caption{SCAS-ADMM for general convex problems}
\label{alg:SCAS-ADMM}
\small
\begin{algorithmic}
   \STATE {\bfseries Initialize:} $(\x_0,\y_0,\bbeta_0)$, a convex set $\XM$
   \FOR{$t=0$ {\bfseries to} $T-1$}
   \STATE Compute $\z_t = \nabla f(\x_t)=\frac{1}{n}\sum_{i=1}^n \nabla f_i(\x_t)$;
   \STATE $\w_0 = \x_t$;
   \STATE $\s = \w_0$;
   \FOR{$m=0$ to $M_t-2$}
   \STATE Randomly select an $i_m$ from $\left\{1,2,\cdots,n\right\}$;
   \STATE $\w_{m+1} = \pi_\XM(\w_m - \eta_t[\nabla f_{i_m}(\w_m) - \nabla f_{i_m}(\w_0) + \z_t+\A^T\bbeta_t+\rho \A^T(\A\w_m+\B\y_t - \c)])$;
   \STATE $\s = \s + \w_{m+1}$;
   \ENDFOR
   \STATE $\x_{t+1} = \frac{1}{M_t} \s$;
   \STATE $\y_{t+1} = \arg\min_{\y} L(\x_{t+1},\y,\bbeta_t)$;
   \STATE $\bbeta_{t+1} = \bbeta_t + \rho(\A\x_{t+1}+\B\y_{t+1}-\c)$;
   \ENDFOR
   \STATE {\bfseries Output: }$\bar{\x}_T = \frac{1}{T}\sum_{t=1}^T \x_t$, $\bar{\y}_T = \frac{1}{T}\sum_{t=1}^T \y_t$
\end{algorithmic}
\end{algorithm}

Compared with SVRG~\cite{DBLP:conf/nips/Johnson013}, the update rule in~(\ref{eq:updateXscasInner}) has an extra vector $\A^T\bbeta_t + \rho \A^T(\A\w_m + \B\y_t -\c)=\rho \A^T\A\w_m + \A^T(\bbeta_t + \rho(\B\y_t-\c))$. If matrix $\A = \0$, which means $\B\y=\c$, $\x$ and $\y$ are independent. Then Algorithm~\ref{alg:SCAS-ADMM} will degenerate to SVRG since we only need to solve the minimization problem about $f(\x)$ and $g(\y)$ separately. We can find that SCAS-ADMM is more general than SVRG since it can solve the minimization problem with more complex equality constraints.


Besides the memory to store $\A$ and $\B$, the memory to store $\z_t$, $\w_0$, $\s$, and $\w_{m+1}$ is only $O(p)$, where $p$ is the number of parameters, i.e., the length of vector $\x$. Furthermore, it only needs some other memory to store $\{\x_t|t=0,1,\cdots,T\}$ and $\{\y_t|t=0,1,\cdots,T\}$. This memory cost is typically small because $T$ is not too large in practice. For example, $T = 15$ is enough for SCAS-ADMM to achieve satisfactory accuracy in our experiments which will be presented in Section~\ref{sec:experiments}. Furthermore, we can also find that SCAS-ADMM does not need to store the historic gradients for all samples which are used in SA-ADMM. Hence, SCAS-ADMM is scalable in terms of storage cost.

\subsubsection{Convergence Analysis}

We call a set $\XM$ is bounded by $D$ if it satisfies: $\underset{\x,\x' \in \XM}{sup} \left\| \x-\x' \right\| \leq D$, where $D$ is a constant.

 Assume we have got $(\x_t,\y_t,\bbeta_t)$, and we define:
\begin{align}
\LM(\x) = &L(\x,\y_t,\bbeta_t).\label{eq:Lx}
\end{align}

We can get the following convergence theorem.
\begin{theorem}\label{theorem:convergenceGeneralConvex}
Assume the optimal solution of~(\ref{eq:lagrangian}) is $(\x_*,\y_*,\bbeta_*)$, $\XM$ is bounded by $D$ and contains $\x_*$, $f(\x)$ and all the functions $\{f_i(\x)\}$ are general convex and $\nu_f$-smooth, and the function $g(\y)$ is convex. We have the following convergence result for Algorithm~\ref{alg:SCAS-ADMM}:
\begin{align}
&\EB \left[f(\bar{\x}_T)+g(\bar{\y}_T)-f(\x_*)-g(\y_*)+\gamma \left\|\A\bar{\x}_T+\B\bar{\y}_T-\c\right\|\right]  \nonumber \\
\leq &\frac{1}{T}\sum_{t=0}^{T-1} \left[\frac{D^2}{2M_t\eta_t} + \eta_t (\nu_\LM^2 D^2 + G_t^2)\right] + \frac{\rho}{2T} \left\| \y_0-\y_* \right\|_\H^2 + \frac{1}{\rho T}(\left\| \bbeta_0 \right\|^2 + \gamma^2),
\end{align}
where $\H = \B^T\B$, $\left\| \x \right\|_\H^2 = \x^T \H \x$, $\gamma > 0$ is a constant, $\nu_\LM$ is the Lipschitz constant of $\LM(\x)$, and $G_t = \left\| \nabla \LM(\x_t)\right\|$.
\end{theorem}

Let $\epsilon_t = \frac{D^2}{2M_t\eta_t} + \eta_t (\nu_\LM^2 D^2 + G_t^2)$. To make $f(\bar{\x}_T)+g(\bar{\y}_T)$ converge to $f(\x_*)+g(\y_*)$, we need to make sure that $\sum_{t=0}^{T-1} \epsilon_t$ is bounded or not too large. By taking $\eta_t = \frac{1}{(\nu_\LM^2D^2 + G_t^2)(t+1)^\delta}$, $M_t = (\nu_\LM^2D^2 + G_t^2)(t+1)^{2\delta}$, we have:
\begin{itemize}
\item If $\delta> 1$, then $\sum_{t=0}^\infty \epsilon_t$ is a constant, which means that $f(\bar{\x}_T)+g(\bar{\y}_T)$ converges to $f(\x_*)+g(\y_*)$ with a convergence rate of $O(\frac{1}{T})$.
\item If $\delta = 1$, then $\sum_{t=0}^{T-1} \epsilon_t = O(\log T)$, which means that $f(\bar{\x}_T)+g(\bar{\y}_T)$ converges to $f(\x_*)+g(\y_*)$ with a convergence rate of $O(\frac{\log T}{T})$.
\end{itemize}
Hence, by choosing $\delta> 1$, we can get a convergence rate $O(\frac{1}{T})$ for our SCAS-ADMM on general convex problems, which is the same as the best convergence rate achieved by existing stochastic ADMM method~(SA-ADMM).



\subsection{Strongly Convex Problems}
In Algorithm~\ref{alg:SCAS-ADMM}, with the increase of $t$, the iteration number of the inner loop $M_t$ needs to be increased and the step size $\eta_t$ needs to be decreased. This might cause large computation when $T$ gets large. We can get a better algorithm when $f(\x)$ in~(\ref{eq:admm}) is strongly convex.

\subsubsection{Algorithm}
When $f(\x)$ is strongly convex, our SCAS-ADMM is briefly presented in Algorithm~\ref{alg:SCAS-ADMM2}. We can find that Algorithm~\ref{alg:SCAS-ADMM2} is similar to Algorithm~\ref{alg:SCAS-ADMM}, but with constant values for $M_t$ and $\eta_t$.
\begin{algorithm}[htb]
\caption{SCAS-ADMM for strongly convex problems}
\label{alg:SCAS-ADMM2}
\small
\begin{algorithmic}
   \STATE {\bfseries Initialize:} $(\x_0,\y_0,\bbeta_0)$, $r = 2\eta - \frac{\eta}{1-\frac{\nu_\LM\eta}{2}}$, $s = \frac{\eta}{1-\frac{\nu_\LM\eta}{2}}$, a convex set $\XM$;
   \FOR{$t=0$ {\bfseries to} $T-1$}
   \STATE Compute $\z_t = \nabla f(\x_t)=\frac{1}{n}\sum_{i=1}^n \nabla f_i(\x_t)$;
   \STATE $\w_0 = \x_t$;
   \STATE $\s = \0$;
   \FOR{$m=0$ to $M-1$}
   \STATE Randomly select an $i_m$ from $\left\{1,2,\cdots,n\right\}$;
   \STATE $\w_{m+1} = \pi_\XM(\w_m - \eta [\nabla f_{i_m}(\w_m) - \nabla f_{i_m}(\w_0) + \z_t+\A^T\bbeta_t+\rho \A^T(\A\w_m+\B\y_t - \c)])$;
   \STATE $\widetilde{\w}_{m+1} = \frac{1}{2\eta}(r \w_m + s \w_{m+1})$;
   \STATE $\s = \s + \widetilde{\w}_{m+1}$;
   \ENDFOR
   \STATE $\x_{t+1} = \frac{1}{M} \s$;
   \STATE $\y_{t+1} = \arg\min_{\y} L(\x_{t+1},\y,\bbeta_t)$;
   \STATE $\bbeta_{t+1} = \bbeta_t + \rho(\A\x_{t+1}+\B\y_{t+1}-\c)$;
   \ENDFOR
   \STATE {\bfseries Output: }$\bar{\x}_T = \frac{1}{T}\sum_{t=1}^T \x_t$, $\bar{\y}_T = \frac{1}{T}\sum_{t=1}^T \y_t$
\end{algorithmic}
\end{algorithm}
\vspace{-0.5cm}

\subsubsection{Convergence Analysis}

\begin{theorem}\label{theorem:convergenceStronglyConvex}
Assume the optimal solution of~(\ref{eq:lagrangian}) is $(\x_*,\y_*,\bbeta_*)$, all the functions $\{f_i(\x)\}$ are general convex and $\nu_f$-smooth, $f(\x)$ is strongly convex and $\nu_f$-smooth, and $g(\y)$ is convex.  We have the following result:
\begin{align}
&\EB \left[f(\bar{\x}_T)+g(\bar{\y}_T)-f(\x_*)-g(\y_*)+\gamma \left\|\A\bar{\x}_T+\B\bar{\y}_T-\c\right\|\right]  \nonumber \\
\leq &\frac{\mu_f}{4T} \left\| \x_0 - \x_* \right\|^2 + \frac{\rho}{2T} \left\| \y_0-\y_* \right\|_\H^2 + \frac{1}{\rho T}(\left\| \bbeta_0 \right\|^2 + \gamma^2),
\end{align}
where $\H = \B^T\B$, and $\gamma > 0$ is a constant.
\end{theorem}

In this case, we can set $M$ and $\eta$ to be constants. Please note that in the proof of Theorem~\ref{theorem:convergenceStronglyConvex}, $M$ and $\eta$ need to satisfy the following conditions: $\eta - \frac{\nu_\LM \eta^2}{2} > 0, (4\nu_\LM^2 + \frac{\mu_f \nu_{\LM}}{2}) \eta + \rho \lambda_1 \leq \mu_{\LM},  \frac{\alpha}{2M\eta}+\frac{2\nu_\LM^2\eta}{2-v_L\eta} \leq \frac{\mu_f}{4}$,
where $\lambda_1$ denotes the maximum eigenvalue of $\A^T\A$, and $\alpha = 1 - \frac{\rho \lambda_1 s}{2} - \frac{\mu_f s}{4}$.
Different from Algorithm~\ref{alg:SCAS-ADMM}, we do not need the convex set $\XM$ in Algorithm~\ref{alg:SCAS-ADMM2} to be bounded or we do not even need such a set for unconstrained problems.

\subsection{Comparison to Related Methods}
We compare our SCAS-ADMM to other stochastic ADMM methods in terms of three key factors: penalty term linearization, convergence rate on general convex problems and memory cost. The matrix inversion $(\frac{1}{\eta_t}\I + \rho\A^T\A)^{-1}$ can be avoided by linearizing the penalty term $\frac{\rho}{2} \left\| \A\x + \B\y - \c \right\|^2$~\cite{DBLP:conf/icml/ZhongK14}. Hence, penalty term linearization can be used to decrease computation cost. The comparison results are summarized in Table~\ref{table:relatedWork}, where SA-IU-ADMM is a variant of SA-ADMM with penalty term linearization. Please note that $\A\in\RB^{l\times p}$, $\B\in\RB^{l\times q}$, $\x\in\RB^{p}$, $\y\in\RB^{q}$, $\c\in\RB^{l}$, $p$ is the number of parameters to learn, and $n$ is the number of training samples.

It is easy to see that only SCAS-ADMM can achieve the best performance in terms of both convergence rate and memory cost. Other methods either achieve only sub-optimal convergence rate, or need more memory than SCAS-ADMM. In particular, SA-ADMM and SA-IU-ADMM need an extra memory as large as $O(np)$ to store the historic gradients for all samples. Typically, $n$ is very large in big data applications. Furthermore, SCAS-ADMM can also avoid the matrix inversion by linearizing the penalty term. Hence, SCAS-ADMM does be salable in terms of both computation cost and memory cost.
\vspace{-0.5cm}
\begin{table*}[htb]
\caption{Comparison to related methods} \label{table:relatedWork}
\small
\begin{center}
\begin{tabular}{|c|c|c|c|}
  \hline
 Method & Penalty term linearization? &Convergence rate& Memory cost\\ \hline \hline
  OADM~\cite{DBLP:conf/icml/WangB12} & NO & $O(1/\sqrt{T})$ & $O(lp+lq)$ \\ \hline
  STOC-ADMM~\cite{DBLP:conf/icml/OuyangHTG13} & NO & $O(1/\sqrt{T})$ & $O(lp+lq)$ \\ \hline
  OPG-ADMM~\cite{DBLP:conf/icml/Suzuki13} & YES &$O(1/\sqrt{T})$ & $O(lp+lq)$ \\ \hline
  RDA-ADMM~\cite{DBLP:conf/icml/Suzuki13} & YES & $O(1/\sqrt{T})$& $O(lp+lq)$ \\ \hline
  OS-ADMM~\cite{DBLP:conf/icml/AzadiS14} & YES & $O(1/\sqrt{T})$& $O(lp+lq)$ \\ \hline
  SA-ADMM~\cite{DBLP:conf/icml/ZhongK14} & NO & $O(1/T)$& $O(np+lp+lq)$ \\ \hline
  SA-IU-ADMM~\cite{DBLP:conf/icml/ZhongK14} & YES & $O(1/T)$& $O(np+lp+lq)$ \\ \hline
  SCAS-ADMM & YES & $O(1/T)$& $O(lp+lq)$ \\ \hline
\end{tabular}
\end{center}
\end{table*}
\vspace{-0.8cm}

\section{Experiments}\label{sec:experiments}
As in~\cite{DBLP:conf/icml/OuyangHTG13,DBLP:conf/icml/AzadiS14,DBLP:conf/icml/ZhongK14}, we evaluate our method on the generalized lasso model~\cite{DBLP:journals/annals/TibshiraniTaylor11} which can be formulated as follows:
\begin{align}\label{eq:gfLasso}
\min_\x \frac{1}{n} \sum_{i=1}^n f_i(\x) + \lambda \left\| \A\x\right\|_1,
\end{align}
where $f_i(\x)$ is the logistic loss, $\A$ is a matrix to specify the desired structured sparsity pattern for $\x$, and $\lambda$ is the regularization hyper-parameter. We can get different models like fused lasso and wavelet smoothing by specifying different $\A$. In this paper, we focus on the graph-guided fused lasso~\cite{DBLP:journals/bioinformatics/KimSX09} which is also used in~\cite{DBLP:conf/icml/ZhongK14}. As in~\cite{DBLP:conf/icml/OuyangHTG13,DBLP:conf/icml/ZhongK14}, we use sparse inverse covariance selection method~\cite{DBLP:journals/jmlr/BanerjeeGd08} to get a graph matrix~(sparsity pattern) $\G$, based on which we can get $\A = [\G;\I]$. In general, both $\G$ and $\A$ are sparse.


We can formulate~(\ref{eq:gfLasso}) with the ADMM framework:
\begin{align} \label{eq:gfLassoADMM}
\min_{\x,\y}~&P(\x,\y) = \frac{1}{n} \sum_{i=1}^n f_i(\x) + g(\y),  \\
s.t.\quad~&\A\x - \y = 0, \nonumber
\end{align}
where $g(\y) = \lambda \left\| \y \right\|_1$.

\subsection{Baselines and Datasets}

Three representative ADMM methods are adopted as baselines for comparison. They are:
\begin{itemize}
\item \emph{Batch-ADMM}~\cite{DBLP:journals/ftml/BoydPCPE11}: The deterministic~(batch) variant of ADMM which uses~(\ref{eq:updataXwithBatch}) to directly update $\x$ by visiting all training samples in each iteration.
\item \emph{STOC-ADMM}~\cite{DBLP:conf/icml/OuyangHTG13}: The stochastic ADMM variant without using historic gradient for optimization, which has a convergence rate of $O(1/\sqrt{T})$ for general convex problems and $O(\log T/T)$ for strongly convex problems.
\item \emph{SA-ADMM}~\cite{DBLP:conf/icml/ZhongK14}: The stochastic ADMM variant by using historic gradient to approximate the full gradient, which has a convergence rate of $O(1/T)$ for general convex problems.
\end{itemize}
Please note that other methods, such as OPG-ADMM, RDA-ADMM and OS-ADMM, are not adopted for comparison because they have similar convergence rate as STOC-ADMM. Furthermore, both theoretical and empirical results have shown that SA-ADMM can outperform other methods like RDA-ADMM and OPG-ADMM~\cite{DBLP:conf/icml/ZhongK14}. The variant of SA-ADMM, SA-IU-ADMM, is also not adopted for comparison because it has similar performance as \mbox{SA-ADMM}~\cite{DBLP:conf/icml/ZhongK14}.

Although the $M_t$ in Algorithm~\ref{alg:SCAS-ADMM} should be increased as $t$ increases, we simply set $M_t = n$ in our experiments because SCAS-ADMM can also achieve good performance with this fixed value for $M_t$. Similarly, we set $M = n$ in Algorithm~~\ref{alg:SCAS-ADMM2}.

As in~\cite{DBLP:conf/icml/ZhongK14}, four widely used datasets are adopted to evaluate our method and other baselines. They are \emph{a9a}, \emph{covertype}, \mbox{\emph{rcv1}} and \emph{sido}. All of them are for binary classification tasks. The detailed information about these datasets can be found in Table~\ref{table:datasets}.
\vspace{-0.5cm}
\begin{table}[ht]
\caption{Information about the datasets} \label{table:datasets}
\begin{center}
\small
\begin{tabular}{|c|c|c|c|}
\hline
Dataset & \#Samples & \#Features  & $\lambda$ \\ \hline \hline
a9a & 32561 & 123  & $10^{-5}$ \\ \hline
covertype & 581012 & 54 & $10^{-5}$ \\ \hline
rcv1 & 20242 & 47236 & $10^{-4}$ \\ \hline
sido & 12678 & 4932 & $ 10^{-4}$ \\ \hline
\end{tabular}
\end{center}
\end{table}
\vspace{-0.5cm}

As in~\cite{DBLP:conf/icml/ZhongK14}, for each dataset we randomly choose half of the samples for training and use the rest for testing. This random partition is repeated for 10 times and the average values are reported. The hyper-parameter $\lambda$ in~(\ref{eq:gfLassoADMM}) is set by using the same values in~\cite{DBLP:conf/icml/ZhongK14}, which are also listed in
Table~\ref{table:datasets}. We adopt the same strategy as that in~\cite{DBLP:conf/icml/ZhongK14} to set the hyper-parameters $\rho$ in~(\ref{eq:lagrangian}) and the stepsize. More specifically, we randomly choose a small subset of 500 samples from the training set, and then choose the hyper-parameters which can achieve the smallest objective value after running 5 data passes for stochastic methods or 100 data passes~(iterations) for batch methods. As in~\cite{DBLP:conf/icml/ZhongK14}, we use $\y(\bar{\x}_T) = \A\bar{\x}_T$ to replace $\bar{\y}_T$ since the methods cannot necessarily guarantee that $\A\bar{\x}_T = \bar{\y}_T$.


All the experiments are conducted on a workstation with 12 Intel Xeon CPU cores and 64G RAM.

\subsection{Convergence Results}


As in~\cite{DBLP:conf/icml/ZhongK14}, we study the variation of the objective value on training set and the testing loss versus the number of \emph{effective passes} over the data. For all methods, one effective pass over the data means $n$ samples are visited. More specifically, one effective pass refers to one iteration in batch ADMM. For stochastic ADMM methods which visit one sample in each iteration, one effective pass refers to $n$ iterations. For SCAS-ADMM, we set $M_t = n$ and each iteration of the outer loop needs to visit $2n$ training samples. Hence, each iteration of the outer loop will contribute two effective passes. Although different methods will visit different numbers of samples in each iteration, we can see that the number of effective passes over the data is a good metric for fair comparison because it measures the computation costs of different methods in a unified way.

Figure~\ref{fig:generalConvex} shows the results for general convex problems with $f_i(\x)$ being the logistic loss. Please note that the number of recorded points on the curve of SCAS-ADMM is half of those for other methods because each iteration of the outer loop of SCAS-ADMM will contribute two effective passes. As stated above, it is still fair to compare different methods with respect to the number of effective passes. In Figure~\ref{fig:generalConvex}, all the points with the same \mbox{x-axis} value from different curves have the same number of effective passes. Hence, for two points with the same \mbox{x-axis} value from any two different curves, the point with smaller y-axis value  is better than the other one. We can find that all the stochastic methods outperform the Batch-ADMM in terms of both training speed and testing accuracy. SCAS-ADMM and SA-ADMM outperform STOC-ADMM, which is consistent with the theoretical analysis about convergence rate. Our SCAS-ADMM can achieve comparable performance as SA-ADMM, which empirically verifies our theoretical result that SCAS-ADMM has the same convergence rate of $O(1/T)$ as SA-ADMM.
\vspace{-0.0cm}
\begin{figure}[htb]
\begin{center}
\subfigure[a9a]{\includegraphics[width=2in]{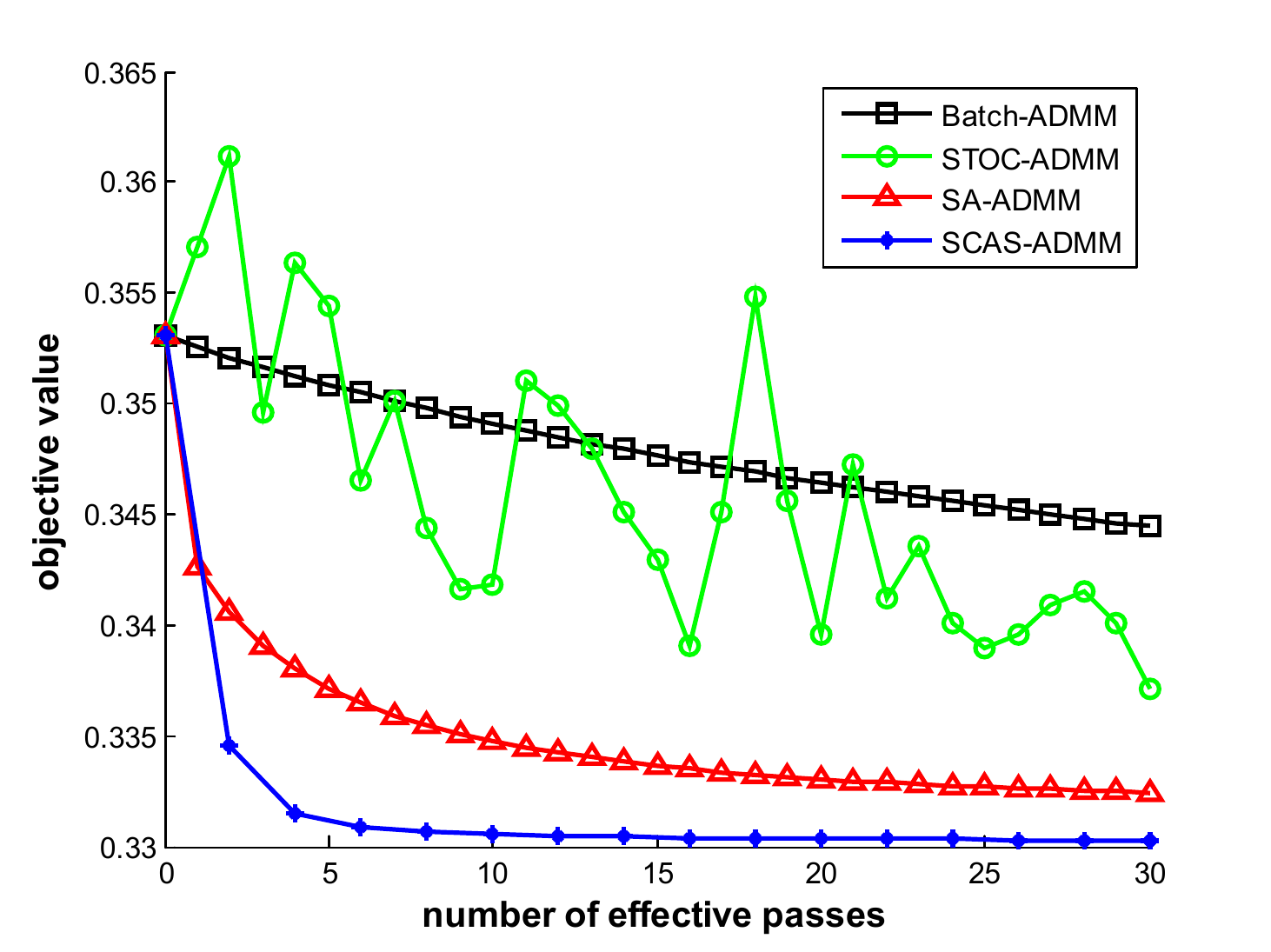}}\hspace{0.0cm}
\subfigure[covertype]{\includegraphics[width=2in]{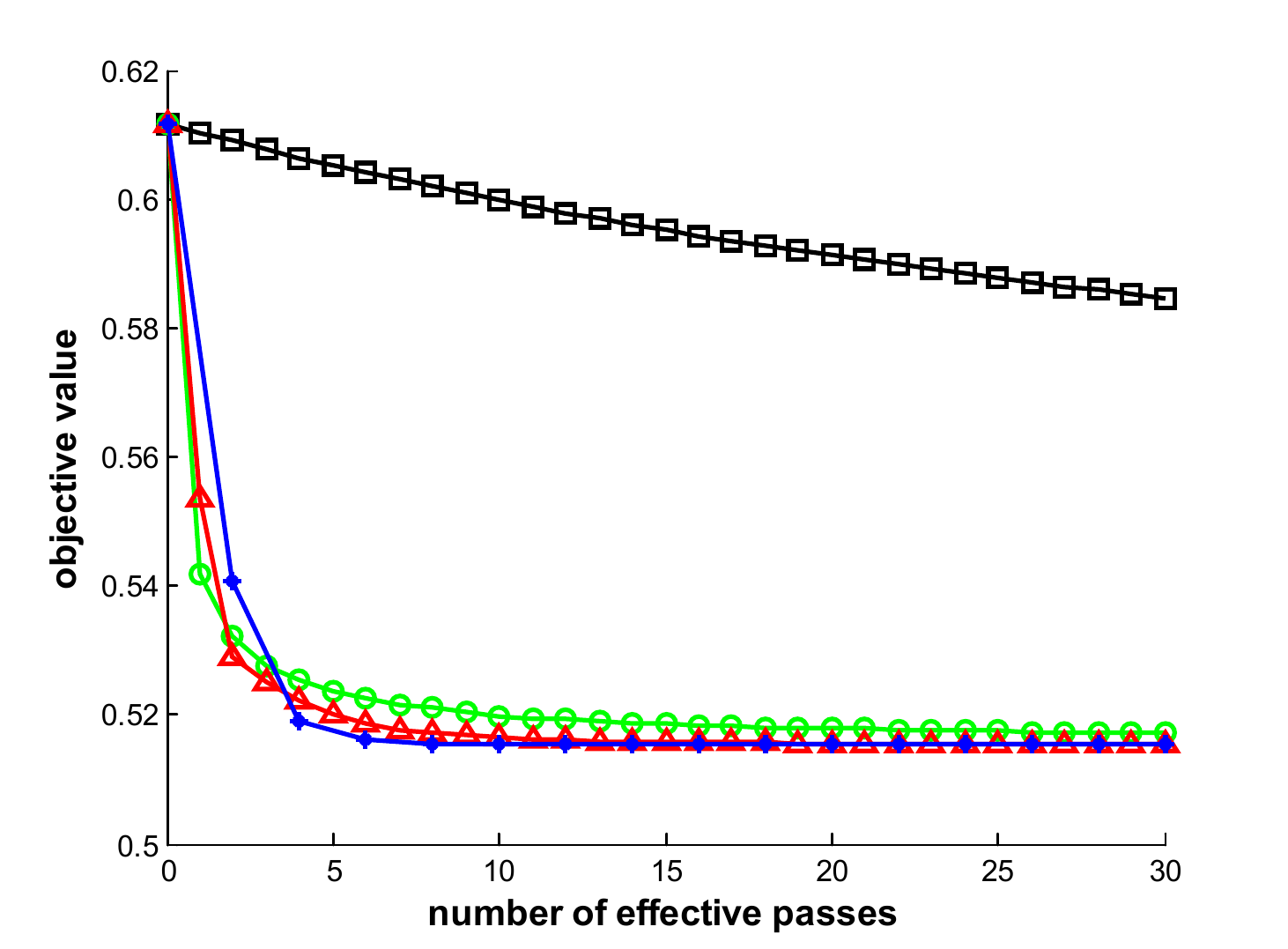}}\hspace{0cm}
\subfigure[rcv1]{\includegraphics[width=2in]{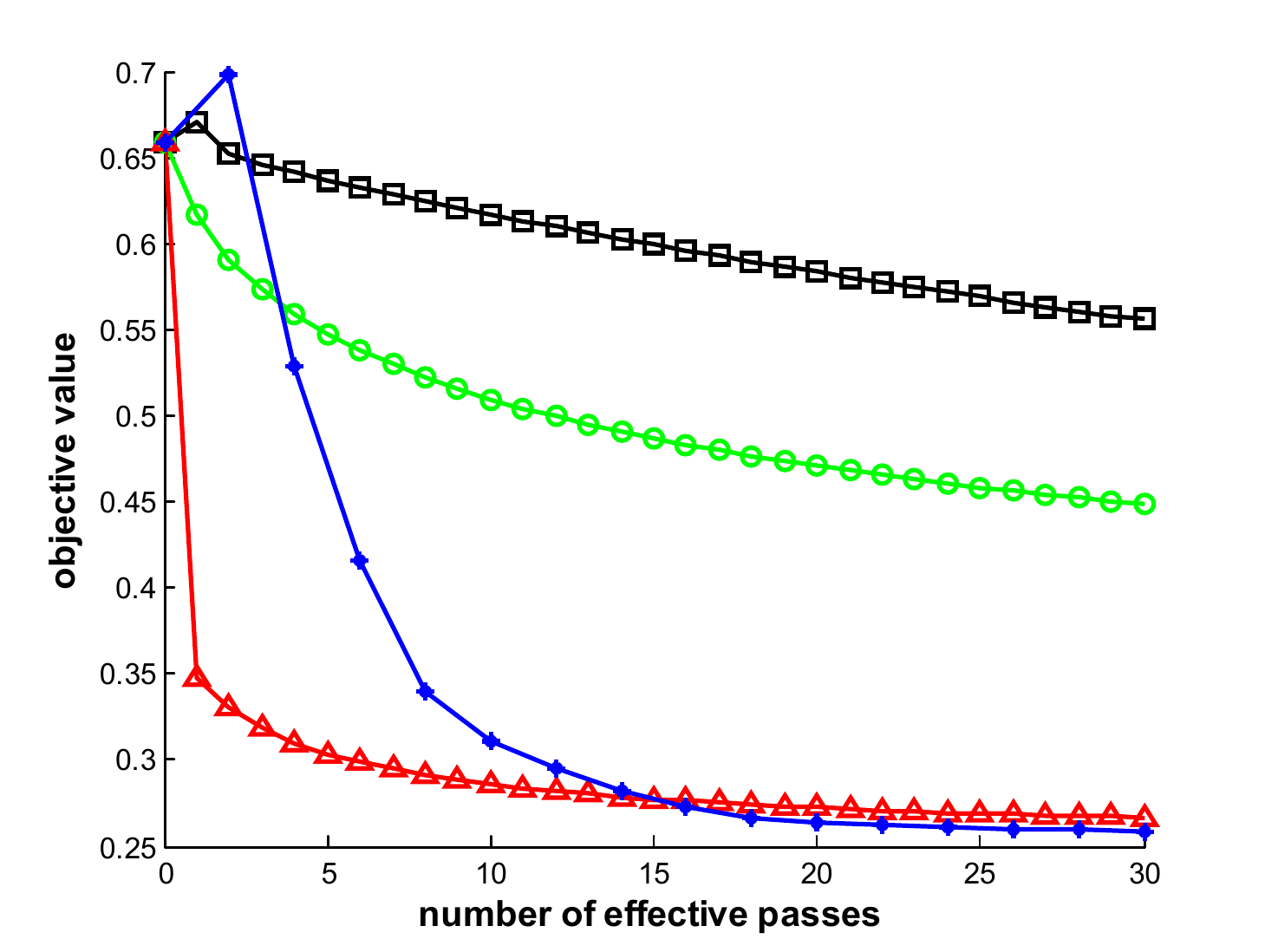}}\hspace{0cm}
\subfigure[sido]{\includegraphics[width=2in]{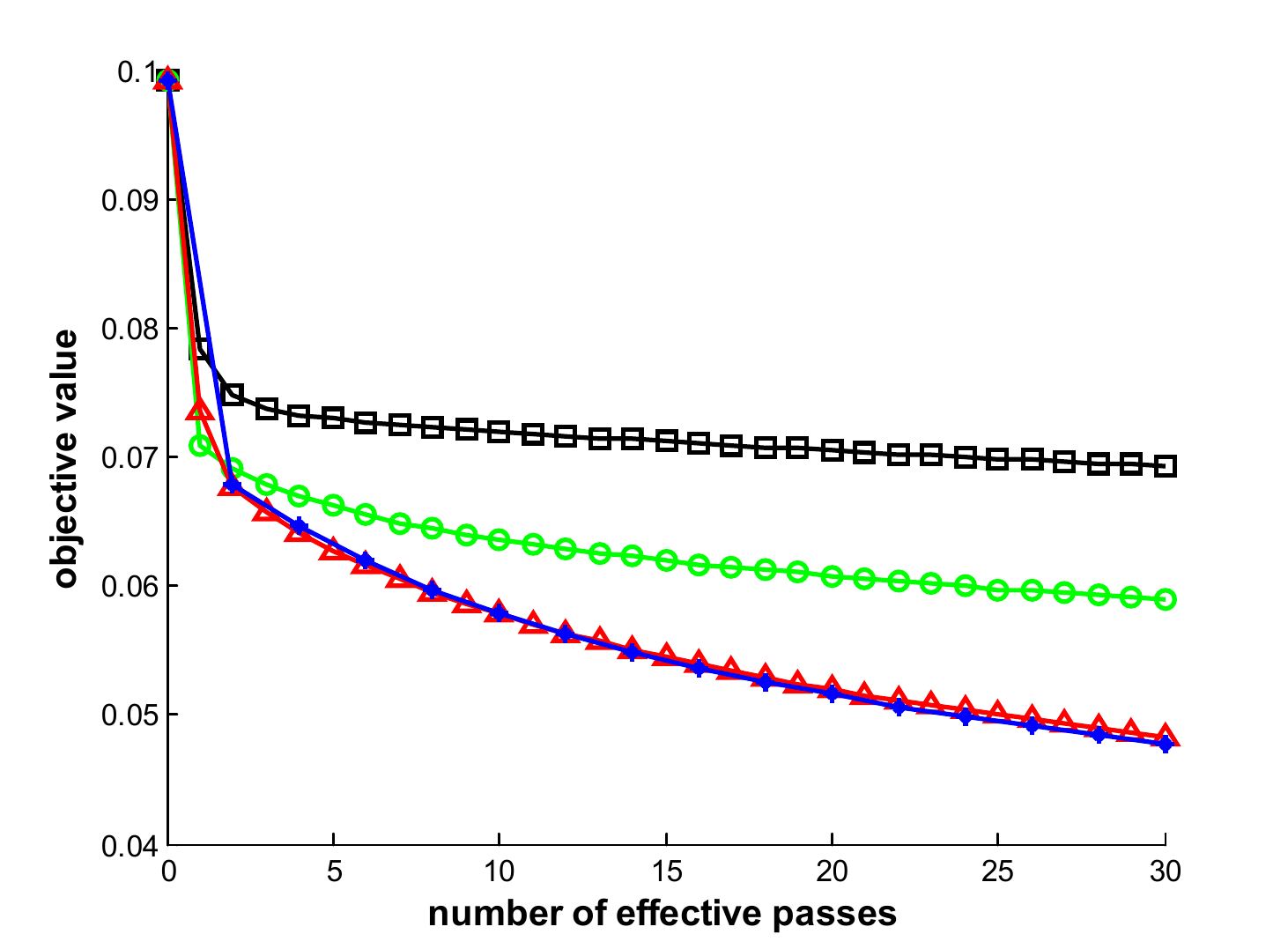}}\hspace{0cm}
\subfigure[a9a]{\includegraphics[width=2in]{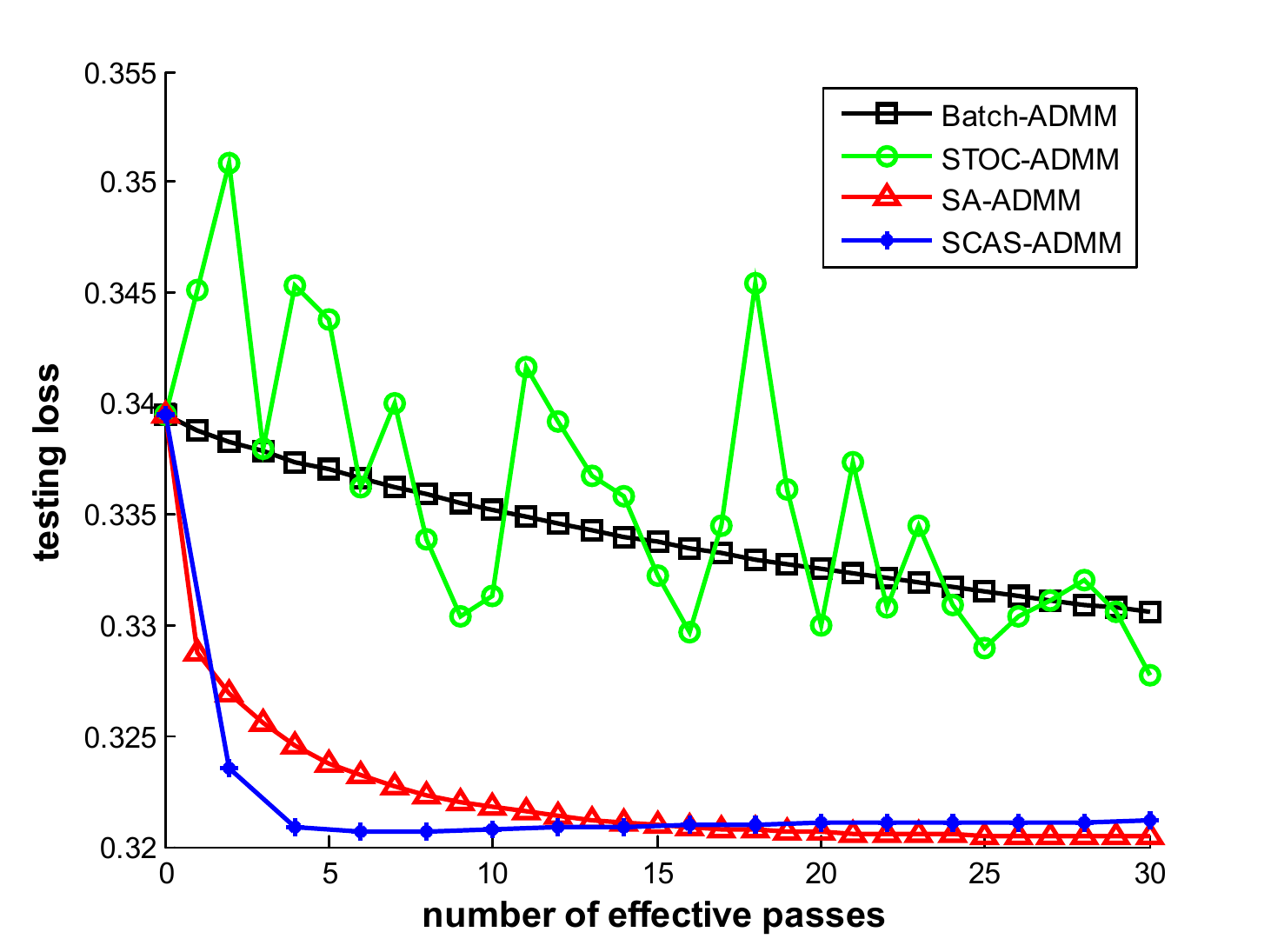}}\hspace{0.0cm}
\subfigure[covertype]{\includegraphics[width=2in]{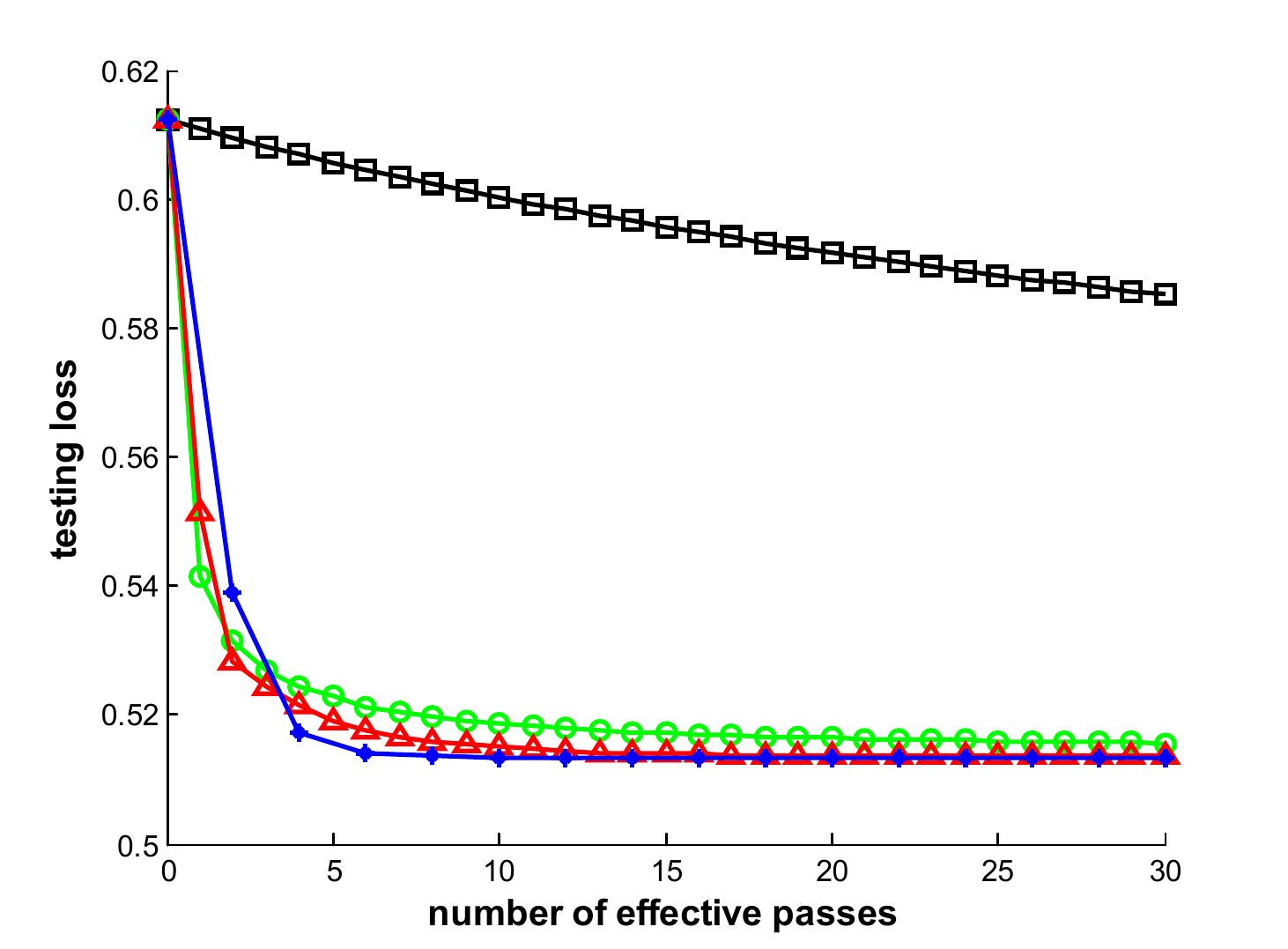}}\hspace{0cm}
\subfigure[rcv1]{\includegraphics[width=2in]{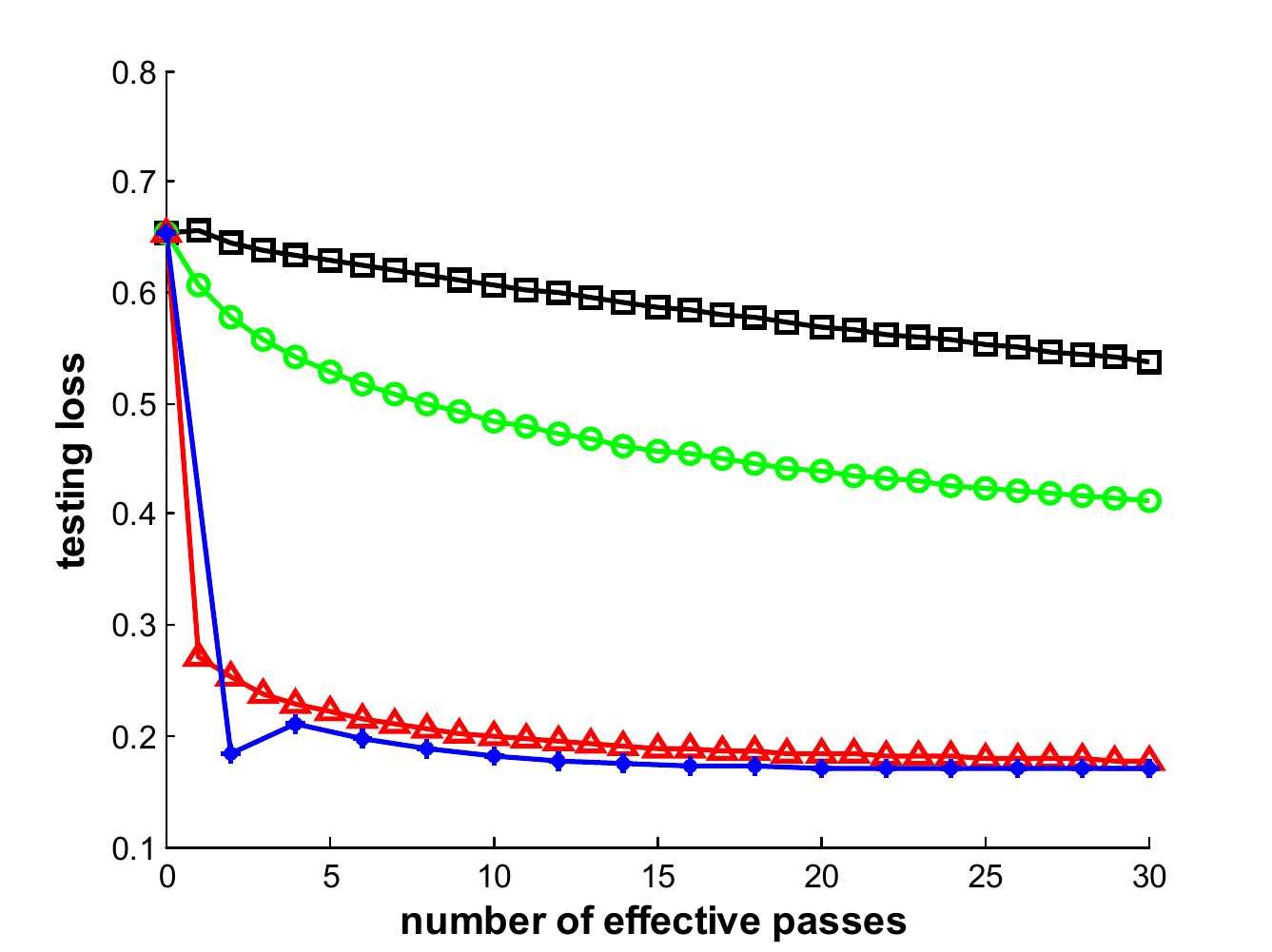}}\hspace{0cm}
\subfigure[sido]{\includegraphics[width=2in]{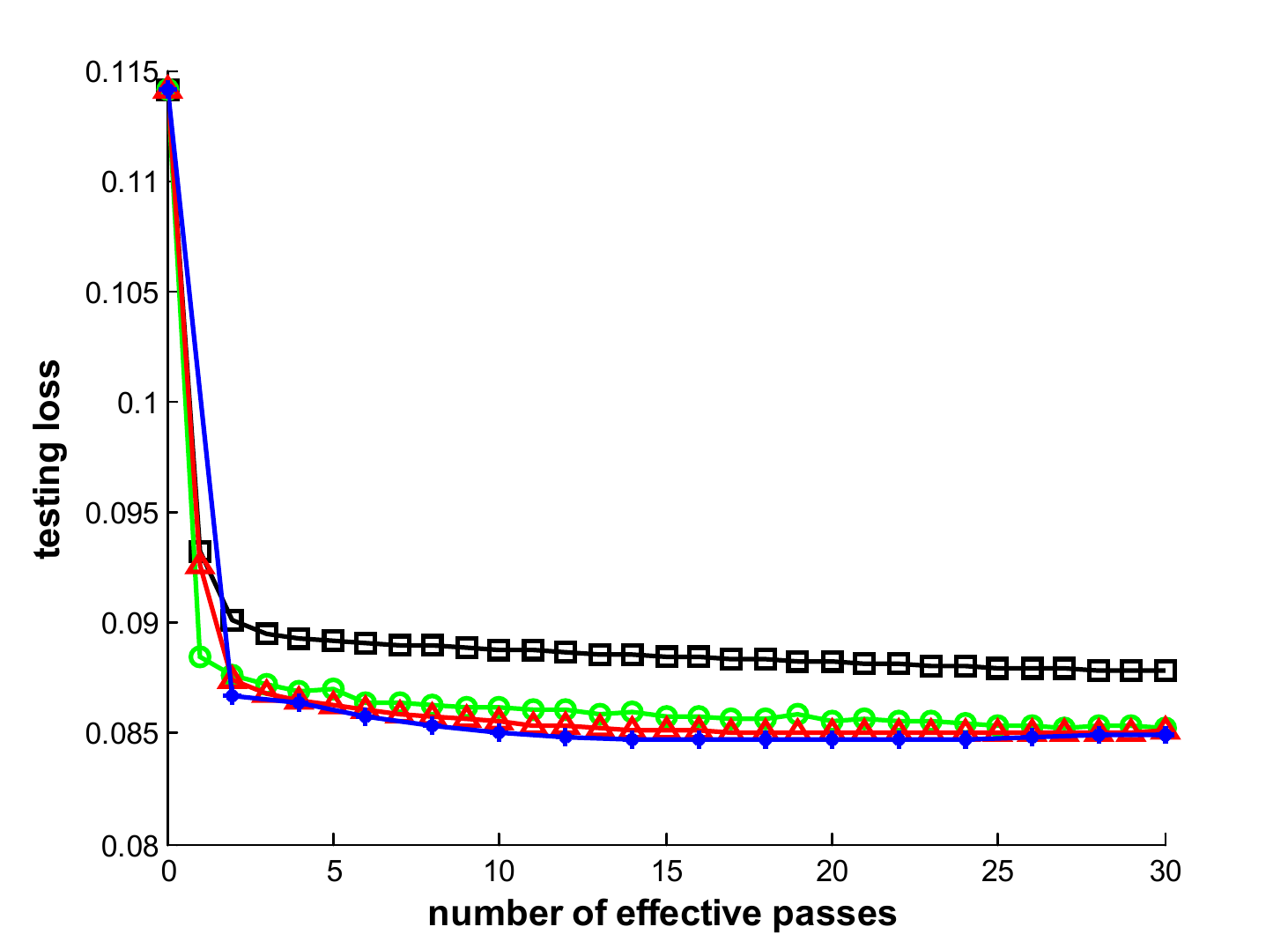}}
\end{center}
\vspace{-0.5cm}
\caption{\small Experiments on four datasets for general convex problems. Top: objective value on training set; Bottom: testing loss.}
\label{fig:generalConvex}
\end{figure}


By adding a small $L_2$ regularization term to the logistic loss, we can get strongly convex problems. Figure~\ref{fig:stronglyConvex} shows the results for strongly convex problems. Once again, we can observe similar phenomenon as that in Figure~\ref{fig:generalConvex}. In particular, our SCAS-ADMM can achieve comparable convergence rate as SA-ADMM.

As for the memory~(storage) cost, it is obvious that SCAS-ADMM needs much less memory than SA-ADMM from the theoretical analysis in Table~\ref{table:relatedWork}. Hence, we do not empirically compare between them.

\vspace{-0.5cm}
\begin{figure}[htb]
\begin{center}
\subfigure[a9a]{\includegraphics[width=2in]{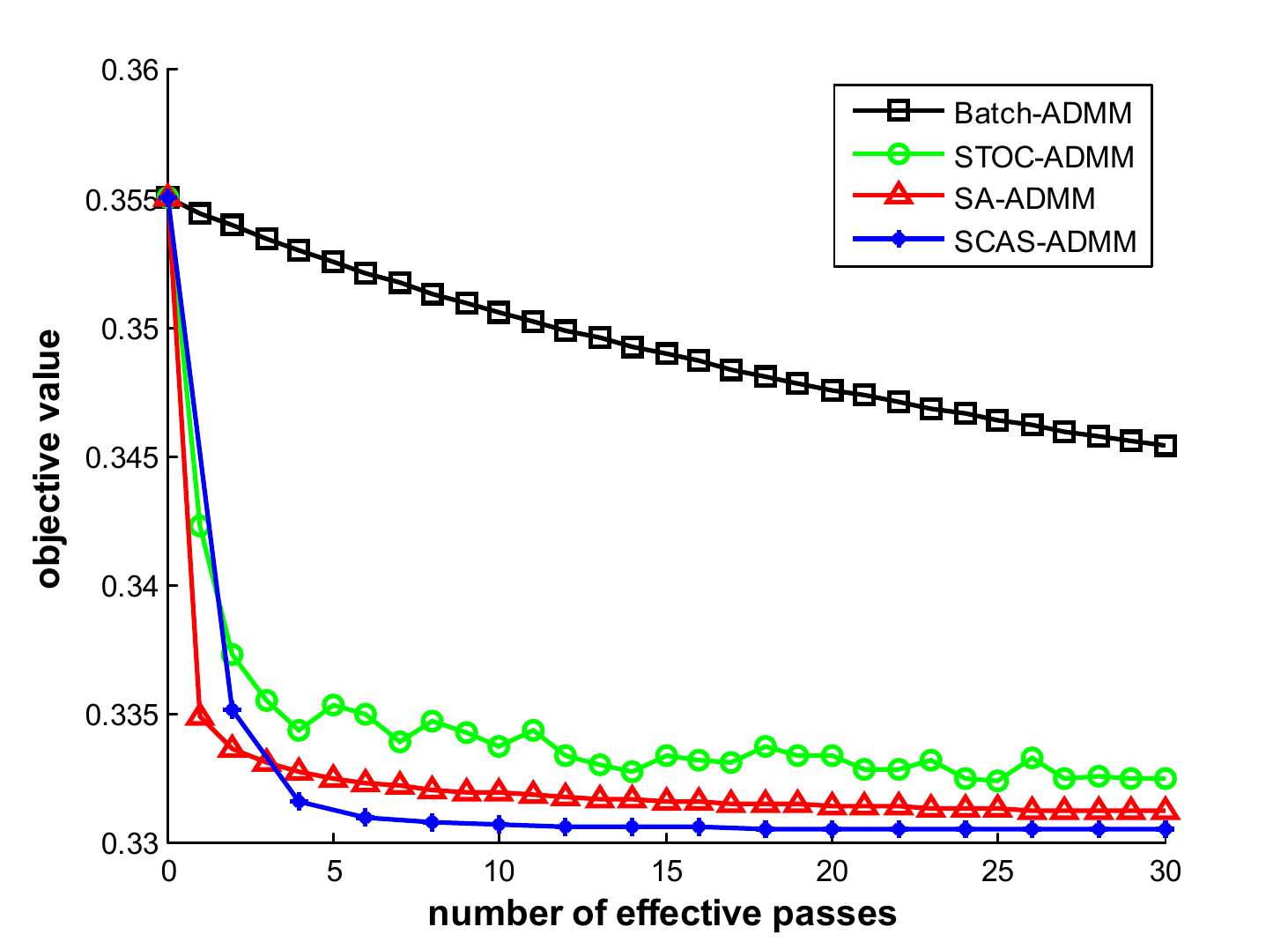}}\hspace{0cm}
\subfigure[covertype]{\includegraphics[width=2in]{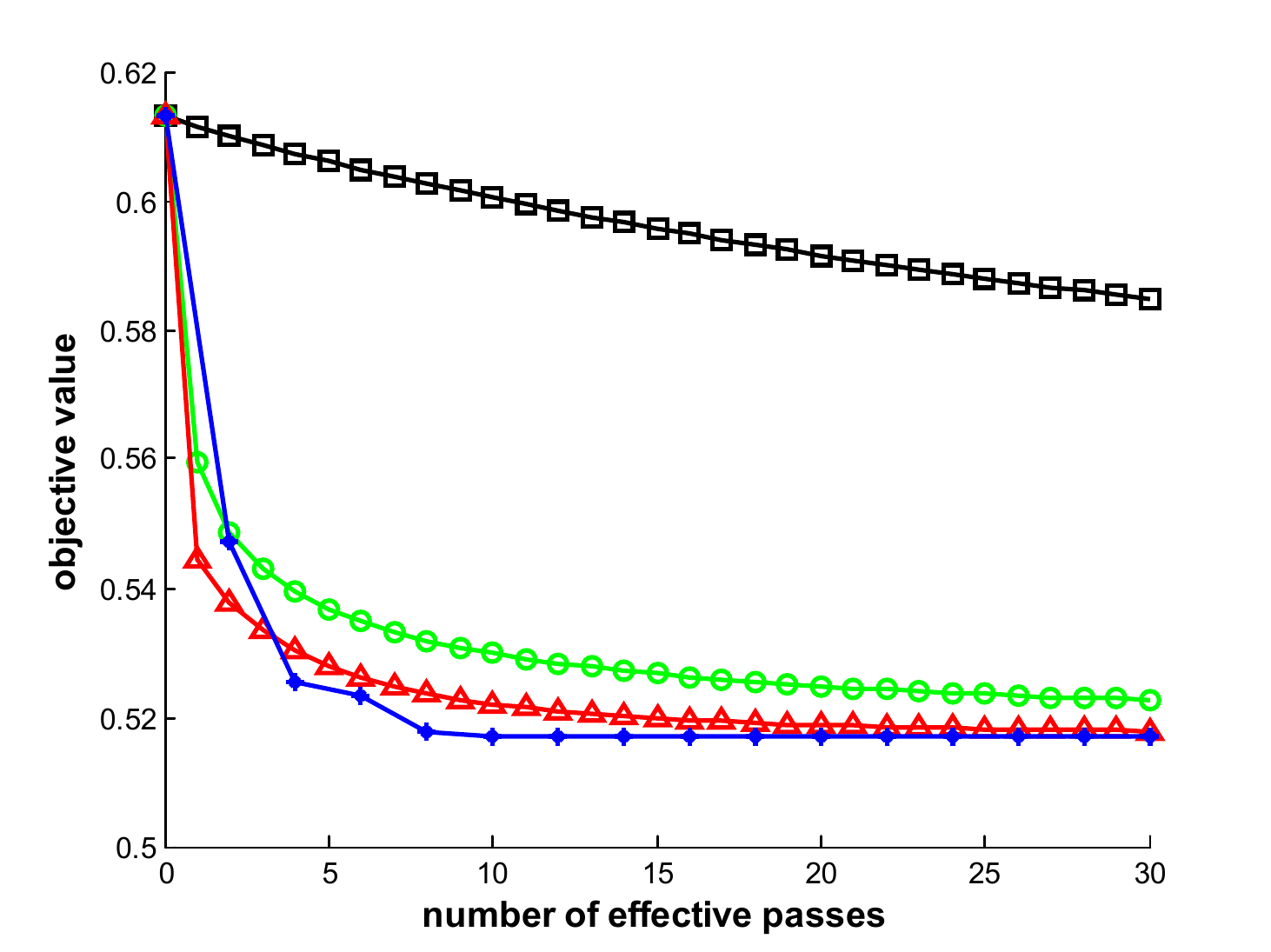}}\hspace{0cm}
\subfigure[rcv1]{\includegraphics[width=2in]{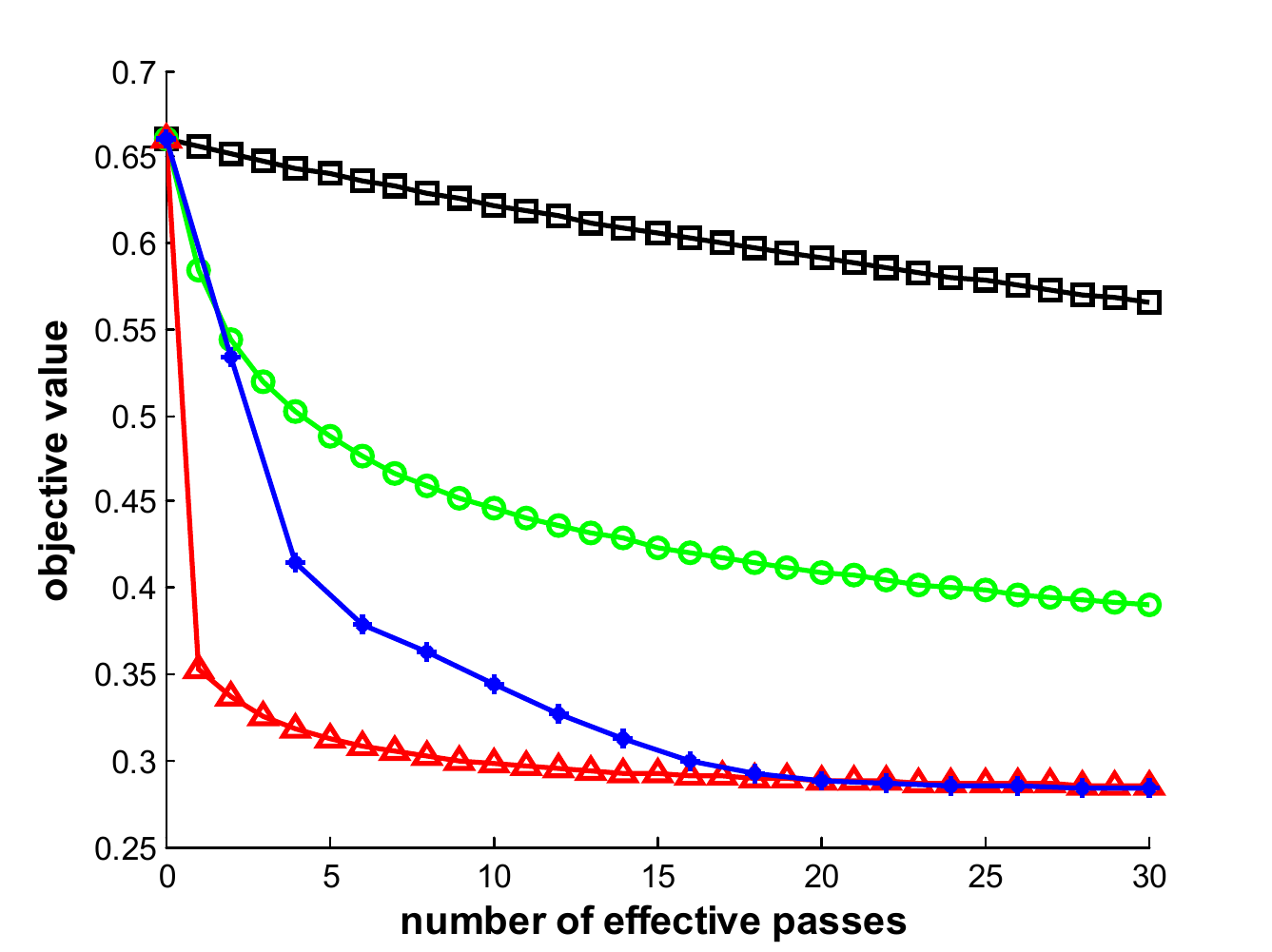}}\hspace{0cm}
\subfigure[sido]{\includegraphics[width=2in]{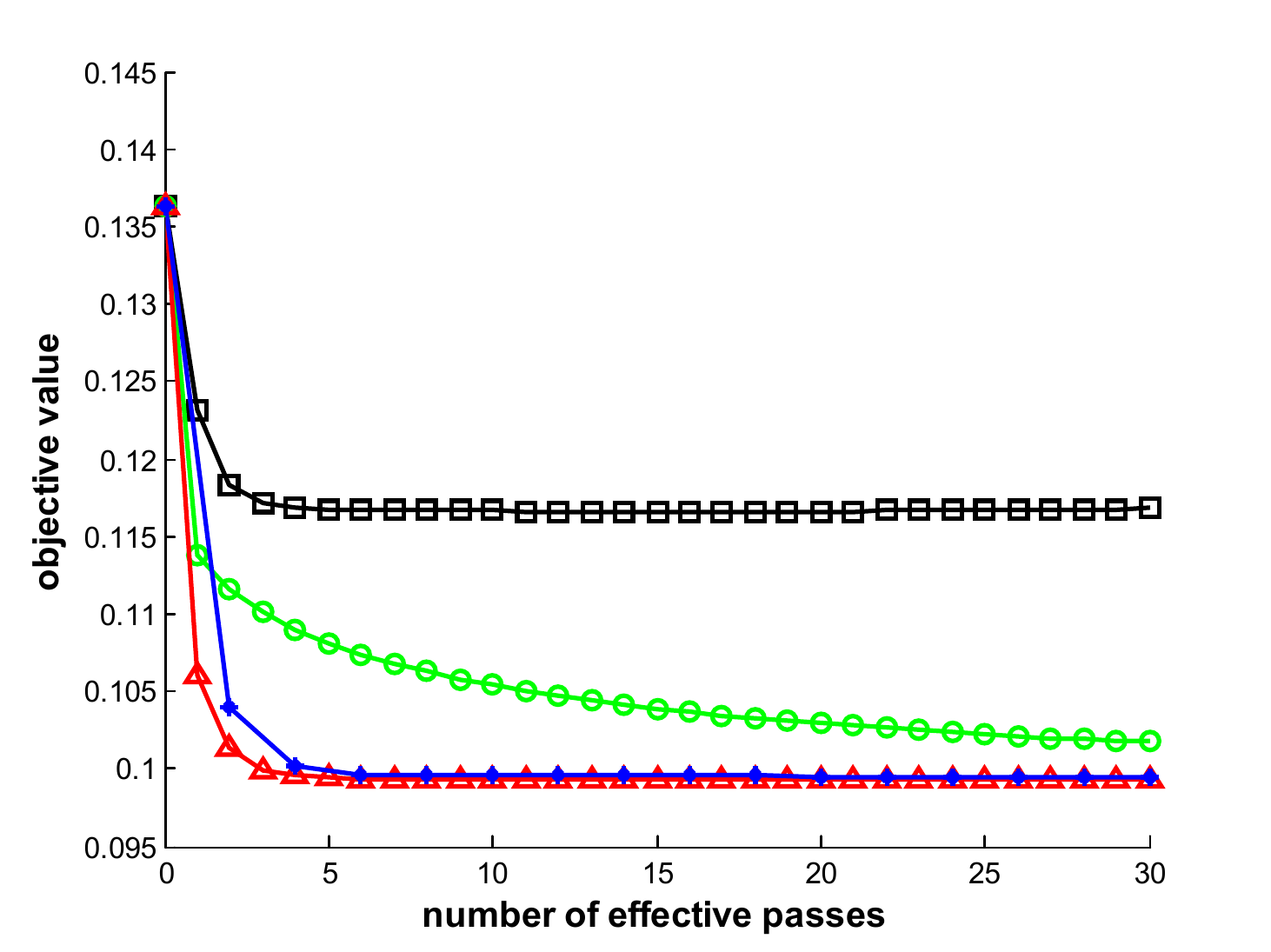}}\hspace{0cm}
\subfigure[a9a]{\includegraphics[width=2in]{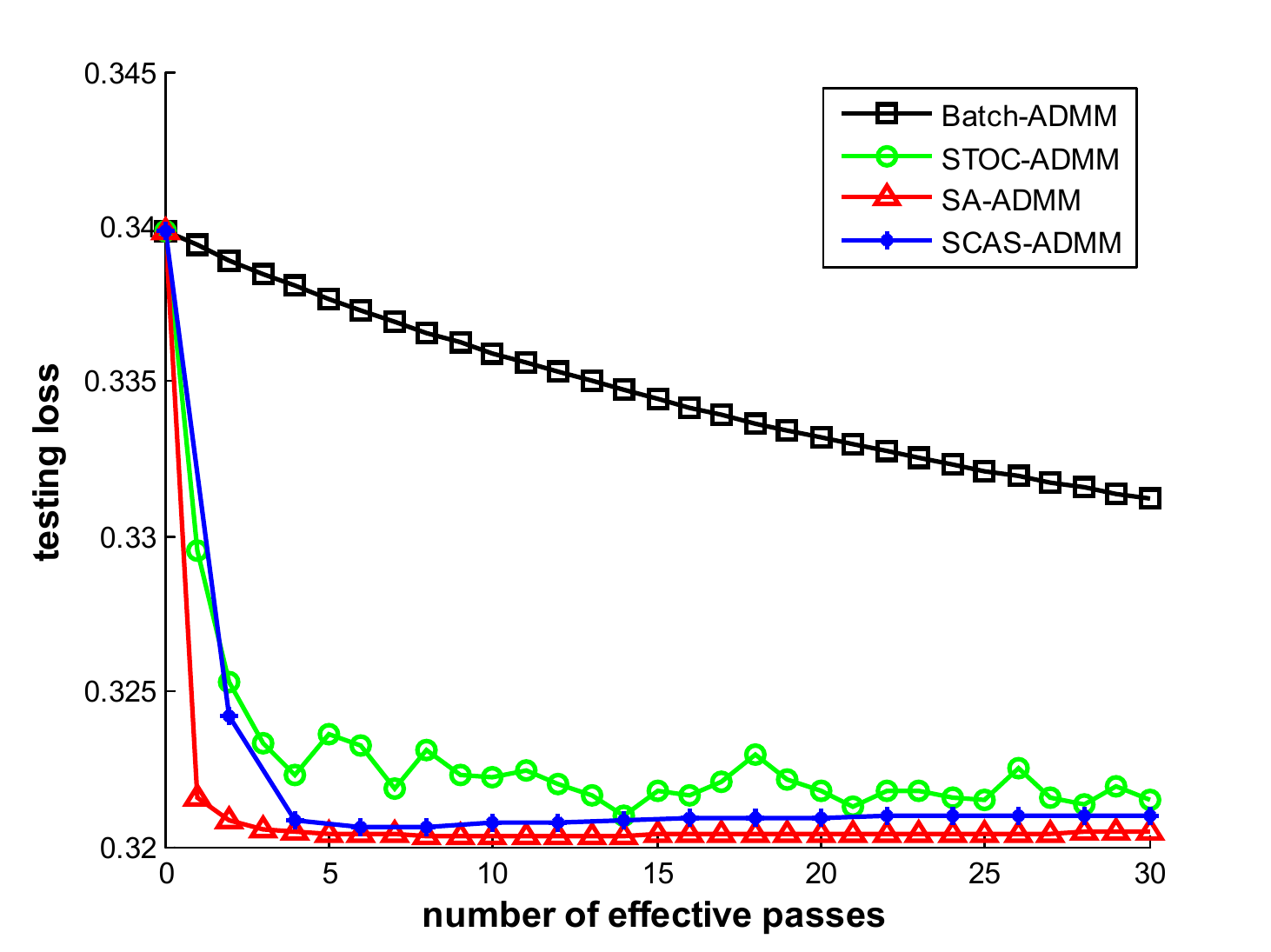}}\hspace{0cm}
\subfigure[covertype]{\includegraphics[width=2in]{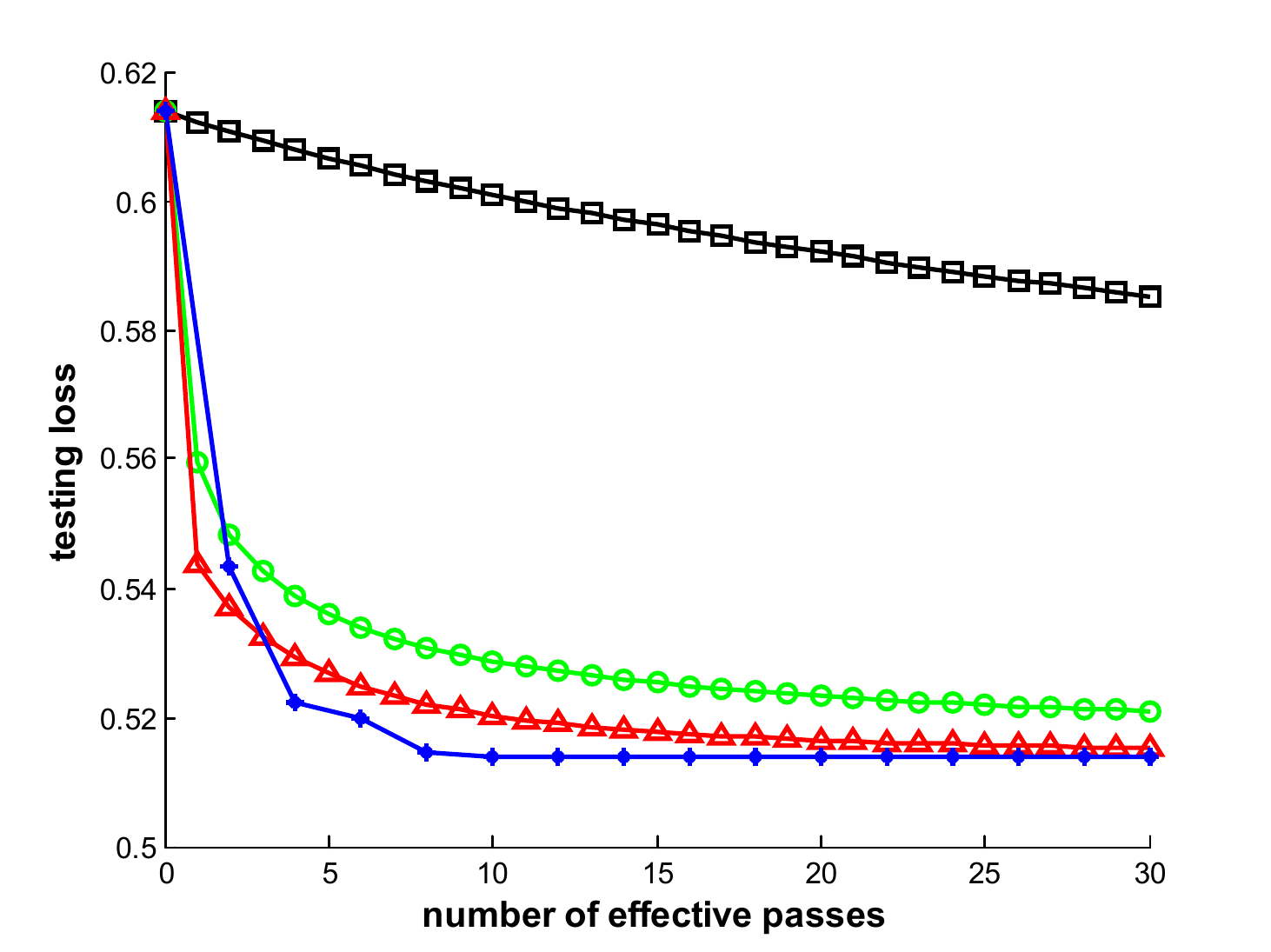}}\hspace{0cm}
\subfigure[rcv1]{\includegraphics[width=2in]{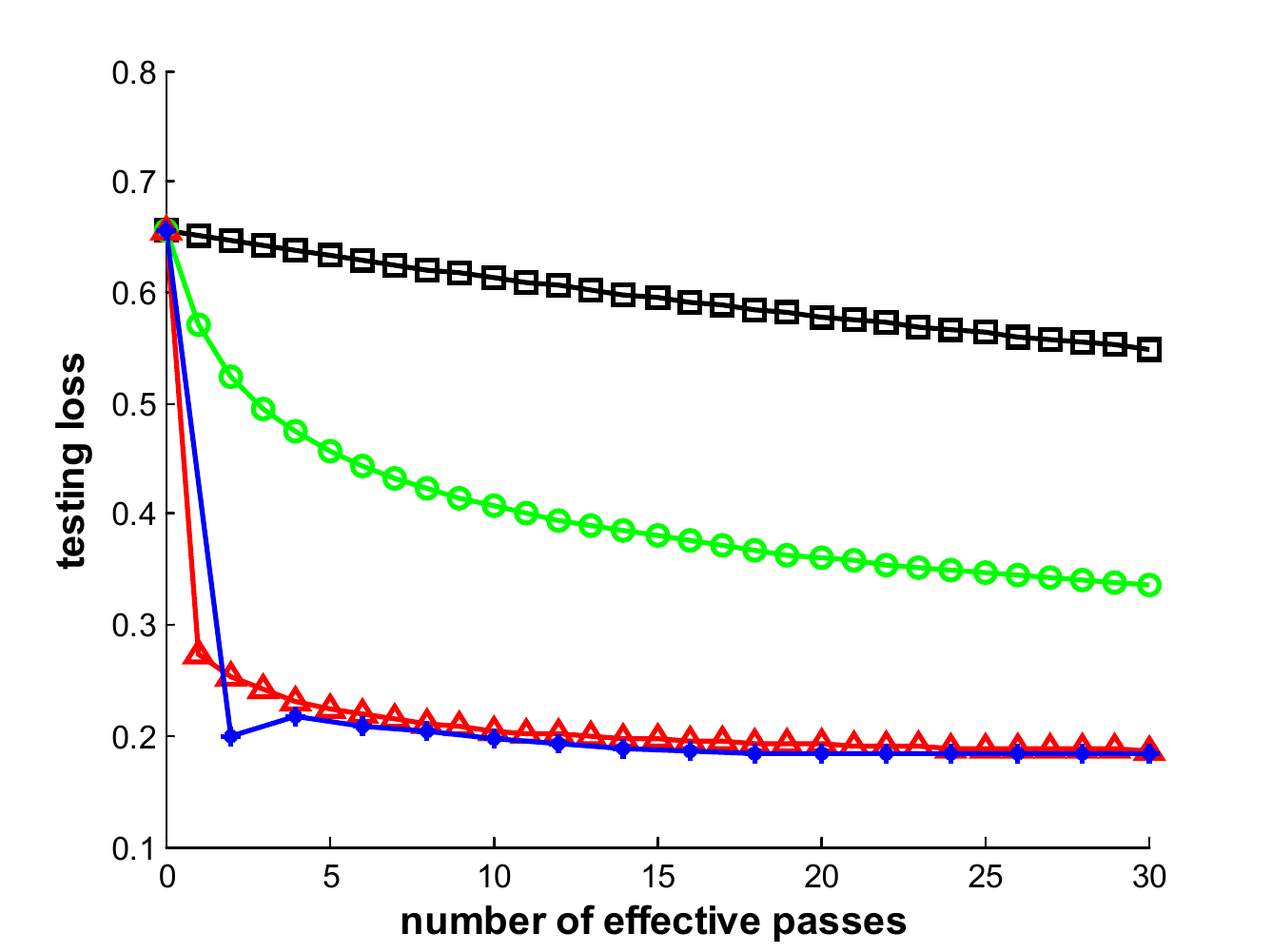}}\hspace{0cm}
\subfigure[sido]{\includegraphics[width=2in]{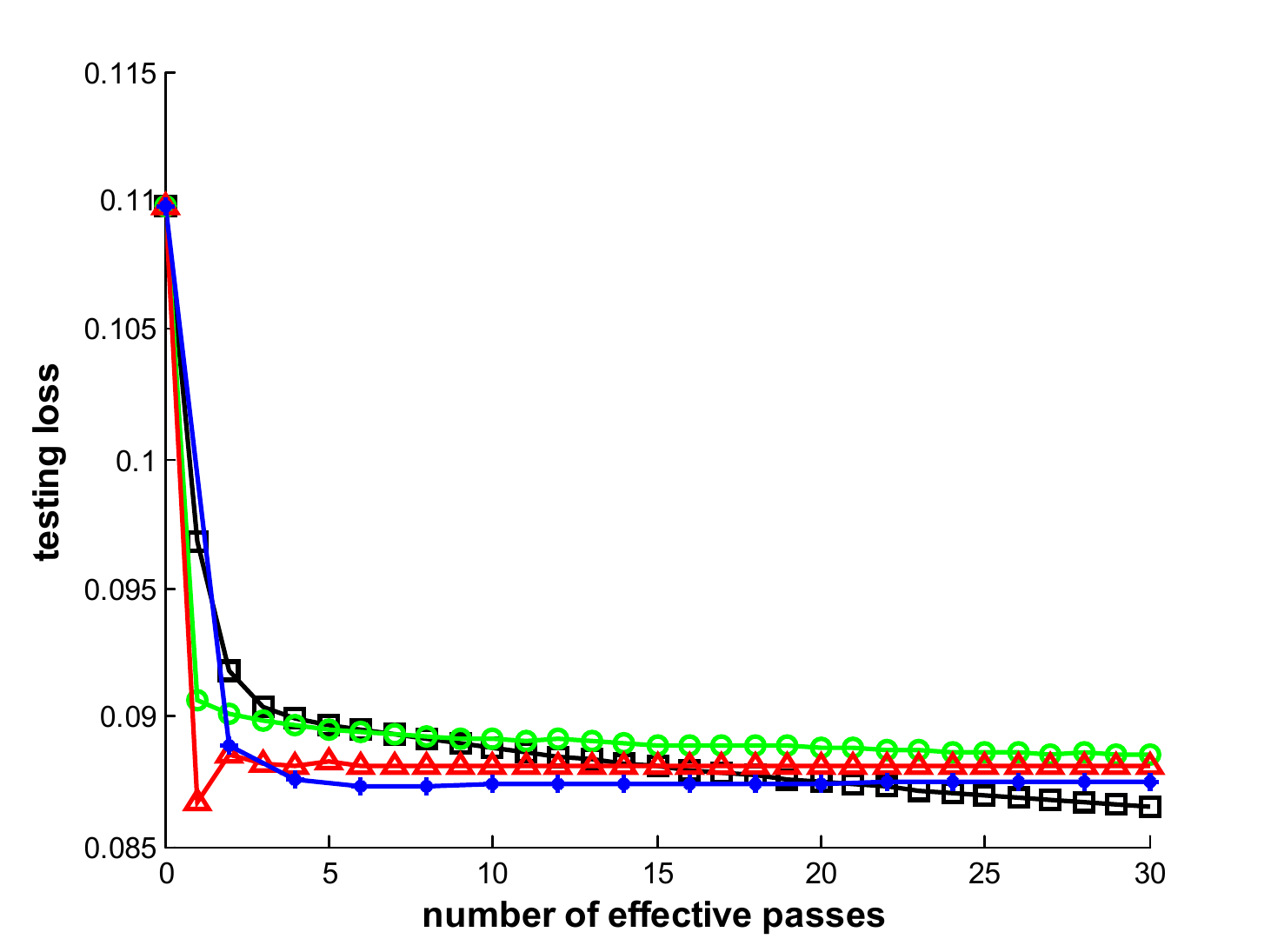}}
\end{center}
\vspace{-0.5cm}
\caption{\small Experiments on four datasets for strongly convex problems. Top: objective value on training set; Bottom: testing loss.}
\label{fig:stronglyConvex}
\end{figure}

\section{Conclusion}\label{sec:conclusion}
In this paper, we have proposed a new stochastic ADMM method called SCAS-ADMM, which can achieve the same convergence rate as the best existing stochastic ADMM method SA-ADMM on general convex problems. Furthermore, it costs much less memory than SA-ADMM. Hence, SCAS-ADMM is scalable in terms of both convergence rate and storage cost.





\bibliography{reference}
\bibliographystyle{unsrtnat}

\newpage
\appendix

\section{Notations for Proof}

We let
\begin{align}
\v_{m,t} &=  \nabla f_{i_m}(\w_m) - \nabla f_{i_m}(\w_0) + \z_t , \label{eq:v_mt}\\
\b_{m,t} &= \bbeta_t + \rho(\A\w_{m} + \B\y_t - \c), \label{eq:b_mt}\\
\p_{m,t} &= \v_{m,t} + \A^T \b_{m,t}.\label{eq:p_mt}
\end{align}

Then the update rule for $\w_{m+1}$ in the inner loop of Algorithm~\ref{alg:SCAS-ADMM} can be rewritten as
\begin{align}\label{eq:w_m}
\w_{m+1} &= \pi_\XM(\w_m - \eta_t(\v_{m,t} + \A^T\b_{m,t})) \nonumber \\
               &= \pi_\XM(\w_m - \eta_t \p_{m,t}).
\end{align}

Assume we have got $(\x_t,\y_t,\bbeta_t)$, and we define:
\begin{align}
\LM(\x) = &L(\x,\y_t,\bbeta_t),\label{eq:Lx} \\
\LM_i(\x) = &f_i(\x)+g(\y_t)+\bbeta_t^T(\A\x+\B\y_t-\c)+\frac{\rho}{2}\left\|\A\x+\B\y_t-\c\right\|^2.
\end{align}

\section{Lemmas for the Proof of Theorem 1}\label{sec:lemmaForTheorem1}

\begin{lemma}\label{lemma:LipschitzLMx}
If $f(\x)$ is $\nu_f$-smooth, then $\exists \nu_\LM > 0$ that makes $\LM(\x)$ be $\nu_\LM$-smooth.
\end{lemma}
\begin{proof}
According to the definition  about $\nu_f$-smooth, $\forall \a,\b$, we have
\begin{align}
\left\| \nabla \LM(\b) - \nabla \LM(\a) \right\| &=\left\| \nabla f(\b) - \nabla f(\a) + \rho \A^T\A(\b-\a)\right\| \nonumber \\
&\leq \left\| \nabla f(\b) - \nabla f(\a)\right\|  + \left\| \rho \A^T\A(\b-\a)\right\| \nonumber \\
& \leq \nu_f \left\| \b - \a \right\| + \left\| \rho \A^T\A(\b-\a)\right\| \nonumber \\
&\leq \nu_f \left\| \b - \a \right\| + \rho \sqrt{\lambda_\A} \left\| \b - \a \right\| \nonumber \\
& = (\nu_f +\rho \sqrt{\lambda_\A}  ) \left\| \b - \a \right\| , \nonumber
\end{align}
where $\lambda_\A \geq 0$ is the largest eigenvalue of $\A^T\A\A^T\A$.

Hence, for any value of $\nu_\LM \geq (\nu_f +\rho \sqrt{\lambda_\A} )$, we can see that $\LM(\x)$ is $\nu_\LM$-smooth.
\end{proof}

We can find that $\nu_\LM$ is only determined by $f(\x)$, matrix $\A$ and the penalty parameter $\rho$, but has nothing to do with $g(\y_t)$, $\y_t$, $\B$ and $\bbeta_t$.

Then, we have the following lemma about the variance of $\p_{m,t}$.
\begin{lemma}\label{lemma:variancePmt}
The variance of $\p_{m,t}$ satisfies:
\begin{align}
\EB ( \left\| \p_{m,t} \right\|^2 ) \leq 2\nu^2_\LM D^2+2G_t^2,
\end{align}
where $D$ is the bound of the domain of $\x$, $\nu_\LM$ is the Lipschitz constant of the function $\LM(\x)$ defined in~(\ref{eq:Lx}), and $G_t = \left\| \nabla \LM(\x_t)\right\|$.
\end{lemma}

\begin{proof}
According to~(\ref{eq:p_mt}), we have
\begin{align}
\p_{m,t} = &\v_{m,t} + \A^T\b_{m,t} \nonumber \\
        = &\nabla f_{i_m}(\w_m) + \A^T\b_{m,t} - f_{i_m}(\w_0) - \A^T\b_{0,t} + \z_t+ \A^T\b_{0,t} \nonumber \\
        = &\nabla \LM_{i_m}(\w_m)-\nabla \LM_{i_m}(\w_0)+\nabla \LM(\w_0) . \nonumber
\end{align}
Then we have:
\begin{align}
\EB ( \left\| \p_{m,t} \right\|^2 ) =  &\frac{1}{n} \sum_{i=1}^n \left\| \nabla \LM_i(\w_m)-\nabla \LM_i(\w_0)+\nabla \LM(\w_0) \right\|^2 \nonumber \\
\leq  &\frac{2}{n} \sum_{i=1}^n \{\left\| \nabla \LM_i(\w_m)-\nabla \LM_i(\w_0)\right\|^2 + \left\| \nabla \LM(\w_0) \right\|^2 \} \nonumber \\
\leq &\{\frac{2}{n} \sum_{i=1}^n \left\|\nabla \LM_i(\w_m)-\nabla \LM_i(\w_0) \right\|^2\} +2\left\| \nabla \LM(\w_0) \right\|^2 \nonumber \\
\leq &2\nu_\LM^2D^2 + 2G_t^2 .\nonumber
\end{align}
Please note that $\w_0 = \x_t$, and we use the Lipschitz definition to get the result:
\begin{align}
&\left\|\nabla \LM_i(\w_m)-\nabla \LM_i(\w_0) \right\|^2  \leq \nu^2_\LM \left\| \w_m - \w_0 \right\|^2  \leq \nu^2_\LM D^2 . \nonumber
\end{align}
\end{proof}

\begin{lemma}\label{lemma:convergenceX}
For the estimation of $\x_{t+1}$, we have the following result:
\begin{align}\label{eq:convergenceX}
\EB \left[f(\x_{t+1})-f(\x)+(\A^T\balpha_{t+1})^T  (\x_{t+1}-\x)\right] \leq \frac{D^2}{2M_t\eta_t} + \eta_t (\nu_\LM^2D^2 + G_t^2),
\end{align}
where $\balpha_{t+1} = \bbeta_t + \rho(\A\x_{t+1}+\B\y_t-\c)$.
\end{lemma}
\begin{proof}
Since $\XM$ is convex, we have:

$\forall \x \in \XM$,
\begin{align}\label{eq:wm_minus_x}
     \left\|\w_{m+1}-\x\right\|^2 &\leq \left\|\w_m-\eta_t \p_{m,t} -\x  \right\|^2  \\
   &= \left\|\w_m-\x\right\|^2 - 2\eta_t \p_{m,t}^T (\w_m-\x) + \eta_t^2\left\| \p_{m,t} \right\|^2 \nonumber.
\end{align}
Furthermore, it is easy to prove that $\EB\left[\v_{m,t}\right]=\nabla f(\w_m)$. Based on the results in Lemma~\ref{lemma:variancePmt}, we can get the expectation on~(\ref{eq:wm_minus_x}):
\begin{align}\label{eq:Exp_wm_minus_x}
     &\EB \left[\left\|\w_{m+1}-\x\right\|^2 \right]\\
\leq &\EB \left[\left\|\w_m-\x\right\|^2\right] - 2\eta_t \EB[(\nabla f(\w_m)+\A^T\b_{m,t})^T (\w_m-\x)] + \eta_t^2 \EB(\left\| \p_{m,t} \right\|^2) \nonumber\\
\leq &\EB \left[\left\|\w_m-\x\right\|^2\right] +\eta_t^2(2\nu_\LM^2D^2 + 2G_t^2) - 2\eta_t \EB[f(\w_m)-f(\x)+(\A^T\b_{m,t})^T (\w_m-\x)]. \nonumber
\end{align}
Summing up~(\ref{eq:Exp_wm_minus_x}) from $m=0$ to $M_t-1$, we can get:
\begin{align}
&2\eta_t\sum_{m=0}^{M_t - 1}\EB[f(\w_m)-f(\x)+(\A^T\b_{m,t})^T (\w_m-\x)] \nonumber \\
\leq &\left\|\w_0-\x\right\|^2 + M_t \eta_t^2(2\nu_L^2D^2 + 2G_t^2) - \EB \left[\left\|\w_{M_t}-\x\right\|^2 \right]\nonumber \\
\leq &\left\|\w_0-\x\right\|^2 + M_t \eta_t^2(2\nu_L^2D^2 + 2G_t^2) \nonumber \\
\leq &D^2 + M_t \eta_t^2(2\nu_\LM^2D^2 + 2G_t^2). \nonumber
\end{align}

We can prove that $f(\w_m)-f(\x)+(\A^T\b_{m,t})^T (\w_m-\x)$ is convex in $\w_m$. Furthermore, we have $\x_{t+1} = \frac{1}{M_t} \sum_{m=0}^{M_t-1} \w_m$. By using the Jensen's inequality, we have:
\begin{align}
&2\eta_t  M_t\EB \left[f(\x_{t+1})-f(\x)+(\A^T\balpha_{t+1})^T (\x_{t+1}-\x)\right] \nonumber \\
\leq &2\eta_t\sum_{m=0}^{M_t - 1}\EB[f(\w_m)-f(\x)+(\A^T\b_{m,t})^T (\w_m-\x)] \nonumber \\
\leq &D^2 + M_t \eta_t^2(2\nu_\LM^2D^2 + 2G_t^2), \nonumber
\end{align}
where $\balpha_{t+1} = \bbeta_t + \rho(\A\x_{t+1}+\B\y_t-\c)$.

Then, we can get:
\begin{align}
&\EB\left[f(\x_{t+1})-f(\x)+(\A^T\balpha_{t+1})^T (\x_{t+1}-\x)\right] \leq \frac{D^2}{2M_t\eta_t} + \eta_t (\nu_\LM^2D^2 + G_t^2).
\end{align}
\end{proof}

According to the results in~\cite{DBLP:conf/icml/ZhongK14}, we have the following Lemma~\ref{lemma:y} and Lemma~\ref{lemma:alpha} about the estimation of $\y_{t+1}$ and $\balpha_{t+1}$.
\begin{lemma}\label{lemma:y}
For the estimation of $\y_{t+1}$, we have:
\begin{align}
&\EB\left[g(\y_{t+1})-g(\y)+(\B^T\balpha_{t+1})^T(\y_{t+1}-\y)\right]  \label{eq:convergenceY}\\
&\leq \frac{\rho}{2}\EB\left[ \left\| \y_t-\y \right\|_\H^2-\left\| \y_{t+1}-\y \right\|_\H^2-\left\| \y_t-\y_{t+1} \right\|_\H^2 \right],\nonumber
\end{align}
where $\balpha_{t+1} = \bbeta_t + \rho(\A\x_{t+1}+\B\y_t-\c)$, $\H = \B^T\B$, and $\left\| \y \right\|_\H^2 = \y^T \H \y $.
\end{lemma}
\begin{lemma}\label{lemma:alpha}
For the estimation of $\balpha_{t+1}$, we have:
\begin{align}
&\EB\left[-(\A\x_{t+1}+\B\y_{t+1}-\c)^T(\balpha_{t+1}-\balpha)\right]  \label{eq:convergenceAlpha}\\
&\leq \frac{1}{2\rho}\EB\left[ \left\| \bbeta_t -\balpha \right\|^2 -\left\| \bbeta_{t+1} - \balpha \right\|^2 \right] + \frac{\rho}{2}\EB\left[ \left\| \y_t - \y_{t+1} \right\|_\H^2 \right], \nonumber
\end{align}
where $\balpha_{t+1} = \bbeta_t + \rho(\A\x_{t+1}+\B\y_t-\c)$, $\H = \B^T\B$, and $\left\| \y \right\|_\H^2 = \y^T \H \y $.
\end{lemma}

The proof of Lemma~\ref{lemma:y} and Lemma~\ref{lemma:alpha} can be directly derived from the results in~\cite{DBLP:conf/icml/ZhongK14}, which is omitted here for space saving.

\section{Proof of Theorem~\ref{theorem:convergenceGeneralConvex}}
\begin{proof}
Let $\u = \left(\begin{array}{c}\x\\\y\\\balpha\\\end{array}\right)$,  $\u_t = \left(\begin{array}{c}\x_t\\\y_t\\\balpha_t\\\end{array}\right)$,  $\bar{\u}_T = \frac{1}{T}\sum_{t=1}^T \u_t$, and $F(\u) = \left(\begin{array}{c}\A^T\balpha\\\B^T \balpha\\-(\A\x+\B\y-\c)\\\end{array}\right)$.

Summing up the equations in~(\ref{eq:convergenceX}), (\ref{eq:convergenceY}) and~(\ref{eq:convergenceAlpha}), we have:
\begin{align}\label{eq:convergenceXYAlpha}
       &\EB \left[ P(\x_{t+1},\y_{t+1})-P(\x,\y) + F(\u_{t+1})^T(\u_{t+1} - \u) \right] \\
\leq &\frac{D^2}{2M_t\eta_t} + \eta_t (\nu_\LM^2D^2 + G_t^2) \nonumber \\
       &+ \frac{\rho}{2}\EB \left[ \left\| \y_t-\y \right\|_\H^2-\left\| \y_{t+1}-\y \right\|_\H^2-\left\| \y_t-\y_{t+1} \right\|_\H^2 \right]\nonumber \\
       &+ \frac{1}{2\rho}\EB \left[ \left\| \bbeta_t - \balpha \right\|^2 -\left\| \bbeta_{t+1} - \balpha \right\|^2 \right] + \frac{\rho}{2}\EB \left[ \left\| \y_t - \y_{t+1} \right\|_\H^2 \right]. \nonumber
\end{align}
It is easy to prove that $ P(\x_{t},\y_{t})-P(\x,\y) + F(\u_{t})^T(\u_{t} - \u)$ is convex in $(\x_t,\y_t)$. Moreover, we have $\bar{\x}_T = \frac{1}{T}\sum_{t=1}^T \x_t$ and $\bar{\y}_T = \frac{1}{T}\sum_{t=1}^T \y_t$. By using the Jensen's inequality, we have:
\begin{align}\label{eq:convergenceXYAlphaJensen}
&P(\bar{\x}_T,\bar{\y}_T)-P(\x,\y) + F(\bar{\u}_T)^T(\bar{\u}_T - \u) \\
\leq &\frac{1}{T} \sum_{t=1}^T \left[P(\x_t,\y_t)-P(\x,\y) + F(\u_t)^T(\u_t - \u) \right].\nonumber
\end{align}
Summing up~(\ref{eq:convergenceXYAlpha}) from $t=0$ to $T-1$, and using the result in~(\ref{eq:convergenceXYAlphaJensen}), we have:
\begin{align}\label{eq:convergenceAll}
&\EB \left[ P(\bar{\x}_T,\bar{\y}_T)-P(\x,\y) + F(\bar{\u}_T)^T(\bar{\u}_T - \u) \right] \nonumber \\
\leq &\frac{1}{T} \sum_{t=0}^{T-1} \EB \left[P(\x_{t+1},\y_{t+1})-P(\x,\y) + F(\u_{t+1})^T(\u_{t+1} - \u) \right] \nonumber \\
\leq &\frac{1}{T}\sum_{t=0}^{T-1} \left[\frac{D^2}{2M_t\eta_t} + \eta_t (\nu_\LM^2D^2 + G_t^2)\right] \nonumber \\
       &+ \frac{\rho}{2T} \left\| \y_0 - \y \right\|_\H^2 + \frac{1}{2\rho T} \left\| \bbeta_0 - \balpha \right\|^2.
\end{align}
The result in~(\ref{eq:convergenceAll}) is satisfied for any $(\x,\y,\balpha)$. In particular, if we take $\x=\x_*$, $\y=\y_*$ and $\balpha  = \gamma \frac{\A\bar{\x}_T+\B\bar{\y}_T-\c}{\left\| \A\bar{\x}_T+\B\bar{\y}_T-\c \right\|}$, we have:
\begin{align}
&\EB \left[ P(\bar{\x}_T,\bar{\y}_T)-P(\x_*,\y_*) + \gamma \left\| \A\bar{\x}_T+\B\bar{\y}_T-\c \right\| \right] \\
\leq &\frac{1}{T}\sum_{t=0}^{T-1} \left[\frac{D^2}{2M_t\eta_t} + \eta_t (\nu_\LM^2D^2 + G_t^2)\right] \nonumber \\
       &+ \frac{\rho}{2T} \left\| \y_0 - \y_* \right\|_H^2 + \frac{1}{\rho T}( \left\| \bbeta_0 \right\|^2 + \gamma^2).\nonumber
\end{align}
\end{proof}

\section{Lemmas for the Proof of Theorem~\ref{theorem:convergenceStronglyConvex}}

\begin{lemma} \label{lemma3.7}
The variance of $\p_{m,t}$ satisfies:
\begin{align}
\forall \x, \hspace{0.2cm} \EB (\left\| \p_{m,t} \right\|^2) &\leq  d_m + \left\| \nabla \LM(\w_m) \right\|^2  \nonumber \\
&\leq 2\nu_\LM^2 \left\| \w_m - \x \right\|^2 + 2\nu_\LM^2 \left\| \w_0 - \x \right\|^2 + \left\| \nabla \LM(\w_m) \right\|^2.
\end{align}
where  $d_m \triangleq \frac{1}{n} \sum_{i=1}^n \left\| \nabla \LM_i(\w_m) - \nabla \LM_i(\w_0) \right\|^2$.
\end{lemma}
\begin{proof}
$\forall \x$
\begin{align}
\EB (\left\| \p_{m,t} \right\|^2) = &\frac{1}{n} \sum_{i=1}^n \left\| \nabla \LM_i(\w_m) - \nabla \LM_i(\w_0) + \nabla \LM(\w_0) \right\|^2 \nonumber \\
= &\left\| \nabla \LM(\w_m) \right\|^2 - \left\| \nabla \LM(\w_m)-  \nabla \LM(\w_0) \right\|^2  + \frac{1}{n} \sum_{i=1}^n \left\| \nabla \LM_i(\w_m) - \nabla \LM_i(\w_0) \right\|^2 \nonumber \\
\leq &\left\| \nabla \LM(\w_m) \right\|^2 + \frac{1}{n} \sum_{i=1}^n \left\| \nabla \LM_i(\w_m) - \nabla \LM_i(\w_0) \right\|^2 \nonumber \\
\leq &\nu_\LM^2 \left\| \w_m - \w_0 \right\|^2 + \left\| \nabla \LM(\w_m) \right\|^2 \nonumber \\
\leq &2\nu_\LM^2 \left\| \w_m - \x \right\|^2 + 2\nu_\LM^2 \left\| \w_0 - \x \right\|^2 + \left\| \nabla \LM(\w_m) \right\|^2 . \nonumber
\end{align}
\end{proof}

\begin{lemma}\label{lemma3.8}
If $\eta - \frac{\nu_\LM \eta^2}{2}>0$, we have the following result for the variance of $\nabla \LM(\w_m)$:
\begin{align}
\left\| \nabla \LM(\w_m) \right\|^2 \leq \frac{1}{\eta - \frac{\nu_\LM \eta^2}{2}} (\LM(\w_m) - \EB [\LM(\w_{m+1})]) + \frac{\nu_\LM \eta}{2 - \nu_\LM\eta} d_m.
\end{align}

\end{lemma}
\begin{proof}
Since $\LM(\w)$ is convex in $\w$, we can get
\begin{align}
\LM(\w_{m+1}) \leq \LM(\w_m) + \nabla \LM(\w_m)^T(\w_{m+1} - \w_m)+\frac{\nu_\LM}{2} \left\| \w_{m+1} - \w_m \right\|^2 .\nonumber
\end{align}
Taking expectation on both sides of the above equation, we get
\begin{align}
\EB [\LM(\w_{m+1})] \leq \LM(\w_m) - \eta \left\| \nabla \LM(\w_m) \right\|^2 +\frac{\nu_\LM \eta^2}{2} \EB (\left\| \p_{m,t} \right\|^2). \nonumber
\end{align}
According to Lemma~\ref{lemma3.7}, we have
\begin{align}
\EB [\LM(\w_{m+1})] \leq \LM(\w_m) - \eta \left\| \nabla \LM(\w_m) \right\|^2 +\frac{\nu_\LM \eta^2}{2} (d_m +  \left\| \nabla \LM(\w_m) \right\|^2). \nonumber
\end{align}
Then we have
\begin{align}
(\eta - \frac{\nu_\LM \eta^2}{2}) \left\| \nabla \LM(\w_m) \right\|^2 \leq \LM(\w_m) - \EB [\LM(\w_{m+1})] + \frac{\nu_\LM \eta^2}{2} d_m .
\end{align}
Choosing a small $\eta$ such that $\eta - \frac{\nu_\LM \eta^2}{2} > 0$, we have
\begin{align}
\left\| \nabla \LM(\w_m) \right\|^2 \leq \frac{1}{\eta - \frac{\nu_\LM \eta^2}{2}} (\LM(\w_m) - \EB [\LM(\w_{m+1})]) + \frac{\nu_\LM \eta}{2 - \nu_\LM \eta} d_m. \nonumber
\end{align}
\end{proof}

\begin{lemma}\label{lemma3.9} We have the following result:
\begin{align}
&\EB (\left\| \w_{m+1} - \x \right\|^2) + 2\eta \nabla \LM(\w_m)^T(\w_m - \x) + \frac{\eta}{1-\frac{\nu_\LM \eta}{2}} (\EB [\LM(\w_{m+1})] - \LM(\x)) \nonumber\\
\leq &\left\| \w_m - \x \right\|^2 + \frac{\eta}{1-\frac{\nu_\LM\eta}{2}} (\LM(\w_m) - \LM(\x)) + \frac{2\eta^2}{2-\nu_\LM\eta}d_m. \nonumber
\end{align}
\end{lemma}
\begin{proof}
From~(\ref{eq:w_m}), we can get $\forall \x$,
\begin{align}
\left\| \w_{m+1} - \x \right\|^2 \leq \left\| \w_m - \x \right\|^2 - 2\eta \p_{m,t}^T(\w_m - \x) + \eta^2 \left\| \p_{m,t} \right\|^2. \nonumber
\end{align}
Then, we have
\begin{align}
\EB ( \left\| \w_{m+1} - \x \right\|^2 ) \leq \left\| \w_m - \x \right\|^2 -  2\eta \nabla \LM(\w_m)^T(\w_m - \x) + \eta^2 \EB (\left\| \p_{m,t} \right\|^2). \nonumber
\end{align}
According to Lemma~\ref{lemma3.7} and Lemma~\ref{lemma3.8}, we have
\begin{align}
\EB (\left\| \w_{m+1} - \x \right\|^2 )&+ 2\eta \nabla \LM(\w_m)^T(\w_m - \x) \leq \left\| \w_m - \x \right\|^2 + \eta^2 d_m + \eta^2 \left\| \nabla \LM(\w_m) \right\|^2 \nonumber \\
&\leq \left\| \w_m - \x \right\|^2 + \eta^2 d_m + \frac{\eta}{1-\frac{\nu_\LM\eta}{2}} (\LM(\w_m) - \EB [\LM(\w_{m+1})]) + \frac{\nu_\LM\eta^3}{2-\nu_\LM\eta} d_m .\nonumber
\end{align}
Then, we can get
\begin{align}
&\EB (\left\| \w_{m+1} - \x \right\|^2) + 2\eta \nabla \LM(\w_m)^T(\w_m - \x) + \frac{\eta}{1-\frac{\nu_\LM \eta}{2}} (\EB [\LM(\w_{m+1})] - \LM(\x)) \nonumber\\
\leq &\left\| \w_m - \x \right\|^2 + \frac{\eta}{1-\frac{\nu_\LM\eta}{2}} (\LM(\w_m) - \LM(\x)) + \frac{2\eta^2}{2-\nu_\LM\eta}d_m . \nonumber
\end{align}
\end{proof}

\begin{lemma}\label{lemma3.10} Let  $\lambda_1$ denote the maximum eigenvalue of $\A^T\A$,
$\forall \w,\x,\y,\bbeta$,
\begin{align*}
\LM(\w)-\LM(\x) \geq f(\w)-f(\x) + (\A^T\b)^T(\w-\x) - \frac{\rho \lambda_1}{2}\left\| \w-\x \right\|^2,
\end{align*}
where $\b = \bbeta + \rho(\A\w+\B\y-\c)$.
\end{lemma}
\begin{proof}
According to the definition $\LM(\w) = f(\w) + g(\y) + \beta^T(\A\w + \B\y - \c) + \frac{\rho}{2}\left\| \A\w + \B\y -\c \right\|^2$, we have
\begin{align}
&\LM(\w)-\LM(\x) \nonumber \\
= &f(\w)-f(\x) + \bbeta^T\A(\w-\x) + \frac{\rho}{2}\left\| \A\w + \B\y -\c \right\|^2 - \frac{\rho}{2}\left\| \A\x + \B\y -\c \right\|^2 \nonumber \\
= &f(\w)-f(\x) + \bbeta^T\A(\w-\x) + \frac{\rho}{2}\left[\left\| \A\w \right\|^2 - \left\| \A\x \right\|^2 + 2(\B\y - \c)^T\A(\w-\x)\right] \nonumber \\
= &f(\w)-f(\x) + (\A^T(\bbeta + \rho(\B\y-\c))^T(\w-\x) + \frac{\rho}{2}(\left\| \A\w \right\|^2 - \left\| \A\x \right\|^2) \nonumber \\
= &f(\w)-f(\x) + (\A^T(\bbeta + \rho(\A\w+\B\y-\c))^T(\w-\x) - \frac{\rho}{2}\left\| \A\w-\A\x \right\|^2 \nonumber \\
= &f(\w)-f(\x) + (\A^T\b)^T(\w-\x) - \frac{\rho}{2}\left\| \A\w-\A\x \right\|^2 \nonumber \\
\geq &f(\w)-f(\x) + (\A^T\b)^T(\w-\x) - \frac{\rho \lambda_1}{2}\left\| \w-\x \right\|^2. \nonumber
\end{align}
\end{proof}

\begin{lemma}\label{lemma16}
If $f(\w)$ is strongly convex and $\mu_f > \rho \lambda_1$, we have the following result
\begin{align}\label{eq:convergencXstrongConvex}
\EB \left[ f(\x_{t+1}) - f(\x) + (\A^T \balpha_{t+1})^T(\x_{t+1} - \x) \right] \leq \frac{\mu_f}{4} (\left\| \x_t - \x \right\|^2 - \EB \left\| \x_{t+1} - \x \right\|^2),
\end{align}
where $\balpha_{t+1}$ is the same as that in Lemma~\ref{lemma:convergenceX}.
\end{lemma}
\begin{proof}
Note that
\begin{align*}
r &\triangleq 2\eta - \frac{\eta}{1-\frac{\nu_\LM\eta}{2}} ,\\
s &\triangleq  \frac{\eta}{1-\frac{\nu_\LM\eta}{2}}.
\end{align*}

Since $f(\w)$ is strongly convex in $\w$, we can prove that $\LM(\w)$ is also strongly convex in $\w$. Then we have
\begin{align}\label{eq:LMstrongConvex}
\nabla \LM(\w_m)^T(\w_m - \x) \geq \LM(\w_m) - \LM(\x) + \frac{\mu_\LM}{2}\left\| \w_m-\x \right\|^2.
\end{align}

By combining the results in Lemma~\ref{lemma3.9} and (\ref{eq:LMstrongConvex}), we have
\begin{align}
&\EB \left\| \w_{m+1} - \x \right\|^2 + \frac{\mu_\LM s}{2} \left\| \w_m - \x \right\|^2 + r \nabla \LM(\w_m)^T(\w_m - \x) + s (\EB [\LM(\w_{m+1})] - \LM(\x)) \nonumber\\
\leq &\left\| \w_m - \x \right\|^2 + \frac{2\eta^2}{2-v_L\eta} d_m. \nonumber
\end{align}

For convenience, we use
\begin{align*}
D_m &\triangleq f(\w_m) - f(\x) + (\A^T \b_{m,t})^T(\w_m - \x) ,
\end{align*}
where the definition of $\b_{m,t}$ is in~(\ref{eq:b_mt}).

And we also have
\begin{align}
&\nabla \LM(\w_m) = \nabla f(\w_m) + \A^T \b_{m,t}, \nonumber \\
&\nabla f(\w_m)^T(\w_m - \x) \geq f(\w_m) - f(\x) + \frac{\mu_f}{2}\left\| \w_m-\x \right\|^2. \nonumber
\end{align}

Then according to Lemma~\ref{lemma3.7} and Lemma~\ref{lemma3.10}, we can get
\begin{align}
&(1 - \frac{\rho s \lambda_1}{2} - \frac{\mu_f s}{4})\EB (\left\| \w_{m+1} - \x \right\|^2 )+ \frac{\mu_\LM s}{2} \left\| \w_m - \x \right\|^2\nonumber \\
&+ r (D_m + \frac{\mu_f}{2} \left\| \w_m - \x \right\|^2) + s \EB (D_{m+1} + \frac{\mu_f}{4} \left\| \w_{m+1} - \x \right\|^2) \nonumber \\
\leq &(1+ \frac{4\nu_\LM^2\eta^2}{2-\nu_\LM \eta})\left\| \w_m - \x \right\|^2 + \frac{4\nu_\LM^2\eta^2}{2-\nu_\LM \eta} \left\| \w_0 - \x \right\|^2, \nonumber
\end{align}
i.e.,
\begin{align}
&(1 - \frac{\rho s \lambda_1}{2} - \frac{\mu_f s}{4})\EB (\left\| \w_{m+1} - \x \right\|^2) \nonumber \\
&+ r (D_m + \frac{\mu_f}{4} \left\| \w_m - \x \right\|^2) + s \EB (D_{m+1} + \frac{\mu_f}{4} \left\| \w_{m+1} - \x \right\|^2) \nonumber \\
\leq &(1+ \frac{4\nu_\LM^2\eta^2}{2-\nu_\LM\eta} - \frac{\mu_f r}{4} -  \frac{\mu_L s}{2}) \left\| \w_m - \x \right\|^2 +  \frac{4\nu_\LM^2\eta^2}{2-\nu_\LM\eta} \left\| \w_0 - \x \right\|^2. \nonumber
\end{align}
Here, $\eta$ need to satisfy the following condition:
\begin{align}
1+ \frac{4\nu_\LM^2\eta^2}{2-\nu_\LM\eta} - \frac{\mu_f r}{4} -  \frac{\mu_\LM s}{2} \leq 1 - \frac{\rho s \lambda_1}{2} - \frac{\mu_f \s}{4}, \nonumber
\end{align}
i.e.,
\begin{align}
(4\nu_\LM^2 + \frac{\mu_f \nu_{\LM}}{2}) \eta + \rho \lambda_1 \leq \mu_{\LM}. \nonumber
\end{align}

Let $\alpha = 1 - \frac{\rho \lambda_1 s}{2} - \frac{\mu_f s}{4}$. We have
\begin{align}
&\alpha \EB \left\| \w_{m+1} - \x \right\|^2 + r \EB (D_m + \frac{\mu_f}{4} \left\| \w_m - \x \right\|^2) + s \EB (D_{m+1} + \frac{\mu_f}{4} \left\| \w_{m+1} - \x \right\|^2) \nonumber \\
\leq &\alpha \EB \left\| \w_m - \x \right\|^2 +  \frac{4\nu_\LM^2\eta^2}{2-\nu_\LM\eta} \left\| \w_0 - \x \right\|^2. \nonumber
\end{align}

Note that $r + s = 2\eta$, and $D_m$ is convex in $\w_m$. We take
$$\widetilde{\w}_{m+1} = \frac{1}{2\eta}(r \w_m + s \w_{m+1}),$$
which is a convex combination of $\w_{m}$ and $\w_{m+1}$.
Then we have
\begin{align}\label{eq:convexCombination}
&\alpha \EB \left\| \w_{m+1} - \x \right\|^2 + 2\eta \EB (\widetilde{D}_{m+1} + \frac{\mu_f}{4} \left\| \widetilde{\w}_{m+1} - \x \right\|^2)
\leq  \alpha \EB \left\| \w_m - \x \right\|^2 + \frac{4\nu_\LM^2\eta^2}{2-\nu_\LM\eta} \left\| \w_0 - \x \right\|^2.
\end{align}
where $\widetilde{D}_{m+1} = f(\widetilde{\w}_{m+1}) - f(\x) + (\A^T \widetilde{\b}_{m+1,t})^T(\widetilde{\w}_{m+1} - \x)$, and $\widetilde{\b}_{m+1,t} = \bbeta + \rho(\A\widetilde{\w}_{m+1}+\B\y-\c)$.

Summing up~(\ref{eq:convexCombination}) from $m=0$ to $M-1$, and taking $\x_{t+1} = \frac{1}{M} \sum_{m=0}^{M-1} \widetilde{\w}_{m+1}$, we have
\begin{align}
&2M\eta \EB (f(\x_{t+1}) - f(\x) + (\A^T \balpha_{t+1})^T(\x_{t+1} - \x) + \frac{\mu_f}{4} \left\| \x_{t+1} - \x \right\|^2)\nonumber \\
\leq & (\alpha + \frac{4M\nu_\LM^2\eta^2}{2-\nu_\LM\eta}) \left\| \w_0 - \x \right\|^2, \nonumber
\end{align}
where $ \balpha_{t+1} = \bbeta + \rho(\A\x_{t+1}+\B\y-\c)$.

Then, we have
\begin{align}
 &\EB \left[ f(\x_{t+1}) - f(\x) + (\A^T \balpha_{t+1})^T(\x_{t+1} - \x) \right] \nonumber \\
 \leq &(\frac{\alpha}{2M\eta}+\frac{2\nu_\LM^2\eta}{2-\nu_\LM\eta})\left\| \x_t - \x \right\|^2 - \frac{\mu_f}{4} \EB (\left\| \x_{t+1} - \x \right\|^2) \nonumber \\
 \leq & \frac{\mu_f}{4} (\left\| \x_t - \x \right\|^2 - \EB (\left\| \x_{t+1} - \x \right\|^2)), \nonumber
\end{align}
where we assume that $\frac{\alpha}{2M\eta}+\frac{2\nu_\LM^2\eta}{2-v_L\eta} \leq \frac{\mu_f}{4}$.
\end{proof}

\section{Proof of Theorem~\ref{theorem:convergenceStronglyConvex}}
\begin{proof}
Let $\u = \left(\begin{array}{c}\x\\\y\\\balpha\\\end{array}\right)$,  $\u_t = \left(\begin{array}{c}\x_t\\\y_t\\\balpha_t\\\end{array}\right)$,  $\bar{\u}_T = \frac{1}{T}\sum_{t=1}^T \u_t$, and $F(\u) = \left(\begin{array}{c}\A^T\balpha\\\B^T \balpha\\-(\A\x+\B\y-\c)\\\end{array}\right)$.

Summing up the equations in~(\ref{eq:convergencXstrongConvex}), (\ref{eq:convergenceY}) and~(\ref{eq:convergenceAlpha}), we have:
\begin{align}
       &\EB \left[ P(\x_{t+1},\y_{t+1})-P(\x,\y) + F(\u_{t+1})^T(\u_{t+1} - \u) \right] \nonumber \\
\leq & \frac{\mu_f}{4} (\left\| \x_t - \x \right\|^2 - \EB \left\| \x_{t+1} - \x \right\|^2)\nonumber \\
       &+ \frac{\rho}{2}\EB \left[ \left\| \y_t-\y \right\|_\H^2-\left\| \y_{t+1}-\y \right\|_\H^2-\left\| \y_t-\y_{t+1} \right\|_\H^2 \right]\nonumber \\
       &+ \frac{1}{2\rho}\EB \left[ \left\| \bbeta_t - \balpha \right\|^2 -\left\| \bbeta_{t+1} - \balpha \right\|^2 \right] + \frac{\rho}{2}\EB \left[ \left\| \y_t - \y_{t+1} \right\|_\H^2 \right]. \nonumber
\end{align}
It is easy to prove that $ P(\x_{t},\y_{t})-P(\x,\y) + F(\u_{t})^T(\u_{t} - \u)$ is convex in $(\x_t,\y_t)$. Moreover, we have $\bar{\x}_T = \frac{1}{T}\sum_{t=1}^T \x_t$ and $\bar{\y}_T = \frac{1}{T}\sum_{t=1}^T \y_t$. By using the Jensen's inequality, we have:
\begin{align}
&P(\bar{\x}_T,\bar{\y}_T)-P(\x,\y) + F(\bar{\u}_T)^T(\bar{\u}_T - \u) \nonumber \\
\leq &\frac{1}{T} \sum_{t=1}^T \left[P(\x_t,\y_t)-P(\x,\y) + F(\u_t)^T(\u_t - \u) \right].\nonumber
\end{align}
Then, we have:
\begin{align}\label{eq:convergenceAllStrongConvex}
&\EB \left[ P(\bar{\x}_T,\bar{\y}_T)-P(\x,\y) + F(\bar{\u}_T)^T(\bar{\u}_T - \u) \right] \nonumber \\
\leq &\frac{1}{T} \sum_{t=0}^{T-1} \EB \left[P(\x_{t+1},\y_{t+1})-P(\x,\y) + F(\u_{t+1})^T(\u_{t+1} - \u) \right] \nonumber \\
\leq &\frac{\mu_f}{4T} \left\| \x_0 - \x \right\|^2 + \frac{\rho}{2T} \left\| \y_0 - \y \right\|_\H^2 + \frac{1}{2\rho T} \left\| \bbeta_0 - \balpha \right\|^2.
\end{align}
The result in~(\ref{eq:convergenceAllStrongConvex}) is satisfied for any $(\x,\y,\balpha)$. In particular, if we take $\x=\x_*$, $\y=\y_*$ and $\balpha  = \gamma \frac{\A\bar{\x}_T+\B\bar{\y}_T-\c}{\left\| \A\bar{\x}_T+\B\bar{\y}_T-\c \right\|}$, we have:
\begin{align}
&\EB \left[ P(\bar{\x}_T,\bar{\y}_T)-P(\x_*,\y_*) + \gamma \left\| \A\bar{\x}_T+\B\bar{\y}_T-\c \right\| \right] \nonumber \\
\leq &\frac{\mu_f}{4} \left\| \x_0 - \x_* \right\|^2 + \frac{\rho}{2T} \left\| \y_0 - \y_* \right\|_\H^2 + \frac{1}{\rho T}( \left\| \bbeta_0 \right\|^2 + \gamma^2).\nonumber
\end{align}
\end{proof}

\end{document}